\documentclass[conference]{IEEEtran}
\IEEEoverridecommandlockouts
\pdfoutput=1
\usepackage{cite}
\usepackage{amsmath,amssymb,amsfonts}
\usepackage{amsmath}
\usepackage{graphicx}
\usepackage{textcomp}
\usepackage{multirow}
\usepackage{multicol}
\usepackage{xcolor}
\usepackage{algorithm}  
\usepackage{booktabs}
\usepackage{algpseudocode}  
\usepackage{amsmath}
\usepackage{threeparttable}
\usepackage{subfigure}
\usepackage{soul}
\usepackage{url}
\newtheorem{df}{\bf Definition}[section]
\newtheorem{as}{\bf Assumption}[section]
\newtheorem{theorem}{\bf Theorem}[section]
\newenvironment{proof}{\quad{\noindent\it Proof Sketch}\quad}{\hfill $\square$\par}

\usepackage{algorithm,algpseudocode}
\makeatletter
\newcommand{\algmargin}{\the\ALG@thistlm}
\makeatother
\newlength{\whilewidth}
\algdef{SE}[parWHILE]{parWhile}{EndparWhile}[1]
{\parbox[t]{\dimexpr\linewidth-\algmargin}{%
		\hangindent\whilewidth\strut\algorithmicwhile\ #1\ \algorithmicdo\strut}}{\algorithmicend\ \algorithmicwhile}%
\algnewcommand{\parState}[1]{\State%
	\parbox[t]{\dimexpr\linewidth-\algmargin}{\strut #1\strut}}

\def\BibTeX{{\rm B\kern-.05em{\sc i\kern-.025em b}\kern-.08em
		T\kern-.1667em\lower.7ex\hbox{E}\kern-.125emX}}
\begin{document}
	
	\title{Learning Scalable Multi-Agent Coordination by Spatial Differentiation for Traffic Signal Control \\
		\thanks{}
	}
	
	\author{\IEEEauthorblockN{Junjia Liu$ ^\dagger $, Huimin Zhang$ ^\ddagger $, Zhuang Fu$ ^\dagger $*\thanks{The first two authors Junjia Liu and Huimin Zhang contributed equally to this paper.
				* Correspondence: zhfu@sjtu.edu.cn (Z. F.),
				Tel.: +86-138 1649 6926 (Z. F.)}, Yao Wang$ ^\dagger $}
		\IEEEauthorblockA{\textit{$ ^\dagger $State Key Laboratory of Mechanical System and Vibration}, \textit{Shanghai Jiao Tong University}\\ \textit{$ ^\ddagger $National Engineering Laboratory for Automotive Electronic Control Technology}, \textit{Shanghai Jiao Tong University} \\
			Shanghai, 200240, China \\
			junjialiu@sjtu.edu.cn, 
			zhanghm1819@sjtu.edu.cn, 
			zhfu@sjtu.edu.cn, 
			sjtuyao@sjtu.edu.cn}
		
	}
	
	\maketitle
	
	\begin{abstract}
		The intelligent control of the traffic signal is critical to the optimization of transportation systems. To achieve global optimal traffic efficiency in large-scale road networks, recent works have focused on coordination among intersections, which have shown promising results. However, existing studies paid more attention to observations sharing among intersections (both explicit and implicit) and did not care about the consequences after decisions. In this paper, we design a multi-agent coordination framework based on Deep Reinforcement Learning method for traffic signal control, defined as \textit{$ \gamma $-Reward} that includes both original \textit{$ \gamma $-Reward} and \textit{$ \gamma $-Attention-Reward}. Specifically, we propose the \textit{Spatial Differentiation} method for coordination which uses the temporal-spatial information in the replay buffer to amend the reward of each action. A concise theoretical analysis that proves the proposed model can converge to Nash equilibrium is given. By extending the idea of Markov Chain to the dimension of space-time, this truly decentralized coordination mechanism replaces the graph attention method and realizes the decoupling of the road network, which is more scalable and more in line with practice. The simulation results show that the proposed model remains a state-of-the-art performance even not use a centralized setting. Code is available in \textit{https://github.com/Skylark0924/Gamma\_Reward}.
		
	\end{abstract}
	
	\begin{IEEEkeywords}
		Multi-agent, coordination mechanism, $ \gamma $-Reward, Deep Reinforcement Learning, spatial differentiation
	\end{IEEEkeywords}
	
	\section{Introduction}
	Traffic congestion has been an increasingly critical matter for decades. It not only leads to an increase in commuting time but also exacerbates noise and environmental pollution issues due to frequent acceleration and deceleration. According to relevant researches, almost all collisions and delays in urban traffic are concentrated on intersections\cite{to2001white}. Unreasonable signal control significantly leads to a waste of traffic resources. Therefore, the key to solve urban congestion is to keep the intersection clear. 
	
	Deep Reinforcement Learning (DRL) methods have been well applied in the traffic signal regulation of single-intersection and shown a better performance than traditional methods\cite{Liang}\cite{wei2018intellilight}\cite{DBLP:journals/corr/abs-1904-08117}, such as Max-pressure\cite{Varaiya2013The}. Recent works began to try to apply DRL algorithms, especially multi-agent Reinforcement Learning (MARL)\cite{bucsoniu2010multi}, to multi-intersection and even large-scale road networks. Different from the single-intersection problem, the intelligent regulation of large-scale road networks needs to achieve synergistic control between various intersections which can be regarded as a  Multi-objective Optimization Problem (MOP) and Markov/Stochastic Games (MGs) with Cooperative Setting\cite{zhang2019multi}. In other words, multiple agents need to coordinate with each other. They need to keep their intersection open, and at the same time, pay attention to the traffic flow status of surrounding or even remote intersections, so that they can ultimately improve the efficiency of the overall road network. The latest research introduced the graph attention network (GAT) to share the observations of real-time traffic volume implicitly with each other, and get an inspired result\cite{Wei2019}. 
	
	\subsection{Related work and Motivation}
	The existing traffic signal control (TSC) methods can be divided into two categories: rule-based methods and learning-based methods. The former transforms the problem into a rule-based optimization problem; the later one seeks control strategy from the traffic flow data.
	
	For the rule-based methods, such as Webster\cite{koonce2008traffic}, GreenWave\cite{T1996The} and Max-pressure\cite{Varaiya2013The}, a traffic signal optimization problem is usually solved under some assumptions like a preset period or fixed cycle-based phase sequence\cite{DBLP:journals/corr/abs-1904-08117}. Webster is used for an isolated intersection and is a widely-used method in TSC. It assumes that the traffic flow is uniform during a certain period and constructs a closed-form equation to generate the optimal cycle length and phase split for a single intersection that minimizes the travel time of all vehicles passing the intersection.
	GreenWave is a classical method in the transportation field to implement coordination,
	which aims to optimize the offsets to reduce the number of stops for vehicles travelling along one specific direction.
	Max-pressure aims to reduce the risk of over-saturation by balancing queue length between adjacent intersections and minimizing the “pressure” of the phases for an intersection. However, the unpractical assumptions in these methods might not lead to excellent performance. 

	Recently, the DRL technique, as a popular learning-based method, has been proposed to control traffic signals due to its capability of online optimization without prior knowledge about the given environment. At present, DRL has been successfully applied to the single-intersection regulation of traffic signal  \cite{wei2018intellilight} by regarding the intersection as an agent. The results of various state-of-the-art DRL algorithms are compared in \cite{mousavi2017traffic}, showing that Deep Q-Networks algorithm is more suitable for the solution of TSC tasks. However, the problem with multiple intersections is still a frontier.
	
	Existing MARL algorithms focus on collaboration among agents and can be divided into centralized and decentralized setting according to various information structures\cite{zhang2019multi}.
	Independent reinforcement learning (IRL), as a fully decentralized setting, directly perform a DRL algorithm for each agent with neither explicit nor implicit information exchange with each other. This method has been applied in multi-intersection TSC problem\cite{xiong2019learning}\cite{zheng2019learning}. However, the environment is shared in MARL, and it changes with the policy and state of each agent\cite{foerster2016learning}. For one of the agents, the environment is dynamic and non-stationary, leads to convergence problems.
	Tan \textit{et.al.}\cite{Tan} compares IRL with Value Decomposition Networks (VDN) and illustrates the disadvantages of IRL. 
	As a centralized method, there exists a central controller in VDN which integrates the value function of each agent to obtain a joint action-value function. The  integration strategy is to add them directly. Moreover, QMIX\cite{rashid2018qmix}, as a extend of VDN, uses state information and integrates them in a nonlinear way and gets a stronger approximation ability than VDN. Both QMIX and VDN are typical centralized MARL algorithms with communication, and this joint-action idea has been already used in TSC\cite{van2016coordinated}.

	Based on these centralized methods, recent TSC studies condense the global scope into a smaller neighborhood\cite{Wei2019}\cite{nishi2018traffic}\cite{wei2019presslight} and use graph convolution network (GCN) to achieve coordination. Colight\cite{Wei2019} introduces the concept of attention mechanism and realized cooperation by integrating observations in a neighborhood implicitly into a hidden state. However, as mentioned in Colight, the neighborhood scope is a constant, so the traffic information among intersections cannot be utilized to determine the range of the neighborhood. Due to the usage of GCN and Multi-head Attention techniques, these methods still need to gather information for centralized computing and betray the intention of distribution. 
	
	In fact, such a central controller does not exist in many applications, apart from those that can easily have a global perspective, like video games\cite{zhang2019multi}. Considering the demand for the scalability in TSC, setting a central controller is impractical. Therefore, we need to seek a compromise solution which is both convergent and scalable. This setting is referred to as \textit{a decentralized one with networked agents}\cite{zhang2019multi}.
	
	\subsection{Main contributions}
	In this paper, we first regard each intersection as a DRL agent and transform the TSC problem into a Markov Decision Process (MDP). Unlike existing work, this paper is aiming at improving method scalability while ensuring a SOTA performance. To achieve this goal, we introduce a structural prior about road networks as an inductive bias and extend the Markov Chain theory to the temporal-spatial domain for the coordination. 
		
	The change of future states and rewards from distant intersections is multiplied by the \textit{spatial discount rate} $\gamma$ (Note that it represents the temporal-spatial discount on multi-agent information, not the ordinary meaning used in temporal difference, and it will be distinguished in detail later) and taken into account when learning the current intersection policy. This information is used as a penalty to correct the calculation of current rewards so that the agents have the ability to collaborate. Due to the various traffic volume of each road, the influence of surrounding intersections may be different. Therefore, the attention mechanism is introduced in this paper to correct the influence weight of surroundings on the current intersection. 
	
	To summarize, our main contributions are as follows:
	\begin{itemize}
		\item We propose a coordination framework, defined as \textit{$\gamma $-Reward}, which can communicate with adjacent intersections and even further in a scalable way by sharing future states and rewards and achieve global optimal control of the TSC problem.
		\item Instead of Multi-head Attention, the \textit{spatial differentiation}  method is proposed to collect the temporal-spatial information in a decentralized way and amend the current reward by recursion.
		\item We just use attention in the neighborhood for distinguishing various significance, and update the attention score parameters in \textit{spatial differentiation} formula by imitating the idea of the target network.
		\item It is found in the test results of various road networks that the \textit{$\gamma $-Reward} series, including original \textit{$ \gamma $-Reward} and \textit{$ \gamma $-Attention-Reward}, maintain a SOTA performance while achieving a better scalability.
	\end{itemize}
	
	\section{Problem Formulation $ \& $ Proposed MARL method}
	In this section, we introduce the basic knowledge of the TSC problem and propose our MARL method.
	\subsection{Preliminary $ \& $ Formulation}
	\begin{itemize}
		\item \textbf{Lane:} Lane is part of a roadway that is designated to be used by a single line of vehicles \cite{wiki:xxx}. There are two kinds of lanes: entering lane and exiting lane\cite{stevanovic2010adaptive}. Each intersection consists of multiple lanes.
		\item \textbf{Phase:} A phase is a combination of movement signals\cite{DBLP:journals/corr/abs-1904-08117}. Figure \ref{Phase}(a) shows eight main directions of vehicles at the intersection. Note that the direction of turning right is usually ignored in these problems since it can execute every time without caring for the traffic signal. The directions in the same phase need not be a conflict which is shown in Figure \ref{Phase}(b). Phase is the unit of TSC, and only one phase can turn green at a time.
		\item \textbf{Neighbor intersection:} The intersections which directly connect to the current intersection. In an informal road network, each intersection usually has at most four neighbors.
		\item \textbf{Waiting vehicle:} If a vehicle on an entering lane has a speed lower than a threshold, then we define it as a \textit{waiting vehicle}, which means it is slowing down to wait for the red light.  
	\end{itemize}
	\begin{figure}[htbp]
		\centering
		\subfigure[Eight main directions in a single intersection]{
			\begin{minipage}[b]{0.6\linewidth}
				\centerline{\includegraphics[width=6cm]{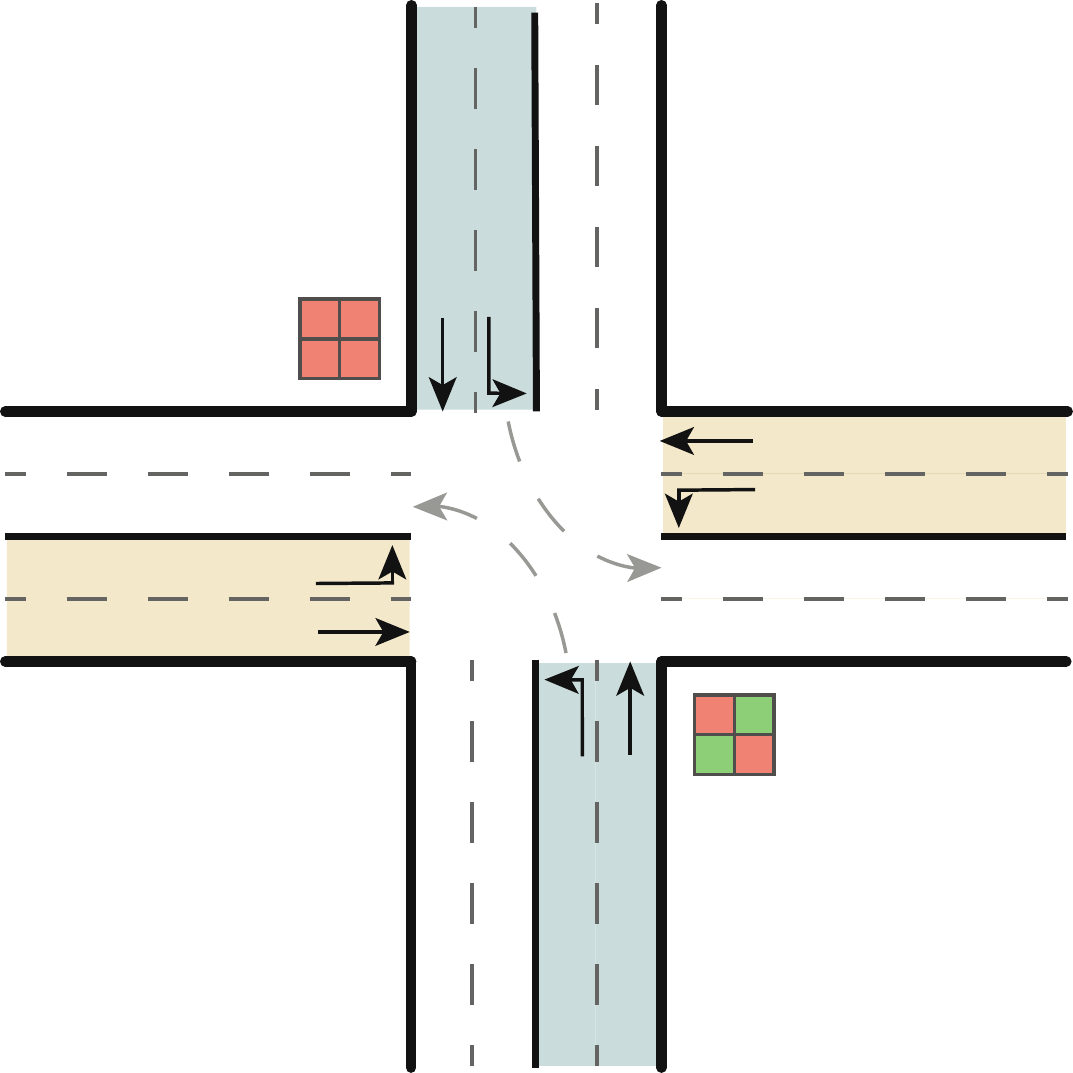}}
			\end{minipage}    
		}
		\subfigure[Eight kinds of primary phases]{
			\begin{minipage}[b]{0.6\linewidth}
				\centerline{\includegraphics[width=4cm]{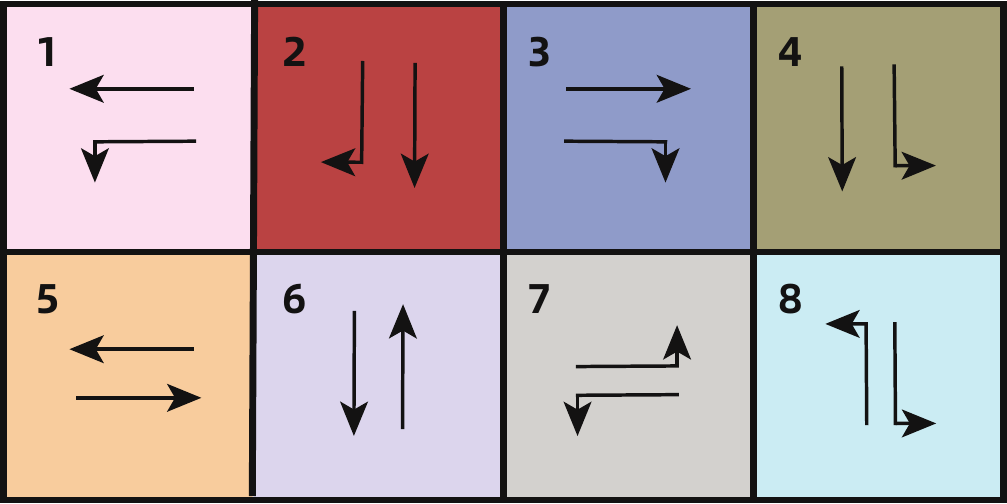}}
			\end{minipage}
		}
		\caption{The foundation of the traffic problem; Phase is the fundamental unit of the TSC problem. The two directions in each phase never conflict.}
		\label{Phase}
	\end{figure}
	
	By using DRL, we regard the TSC problem as MDP. An individual DRL agent controls one of the intersections. They need to observe part of the environmental state $ \mathcal{O} $ and get actions $ \mathcal{A} $ according to these observations to determine which phase in the intersection needs to turn green. The effect of control is fed back from the environment in the form of reward $ \mathcal{R} $. The goal of the DRL agent is to maximize the reward function by continuously exploring and exploiting based on constant interaction with the environment. In this paper, the problem requests to reduce the length of the queue $q_l$ or the travel time $ T_w $ in the road network. To make this problem more suitable for DRL, we can first abstract it into these parts $ \mathcal{<O, A, P, R, \pi,  \gamma >} $:
	\begin{figure}[htbp]
		\centerline{\includegraphics[width=8cm]{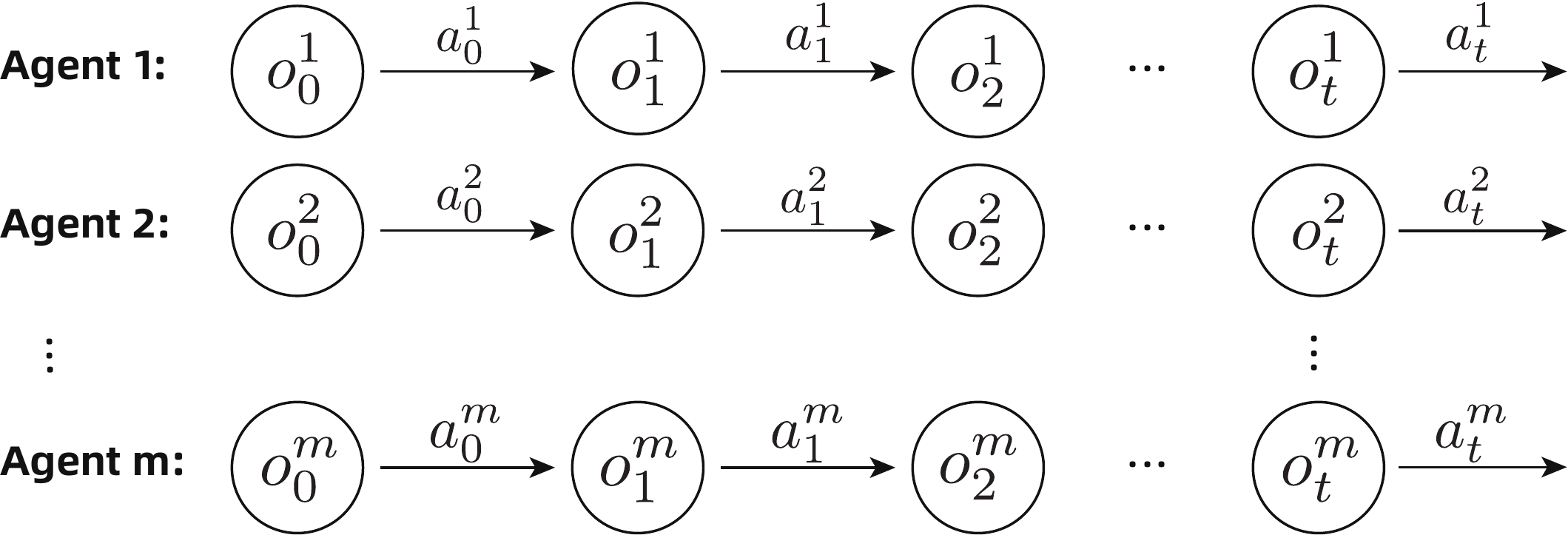}}
		\caption{The TSC problem is regarded as MDP. Each intersection is controlled by a unique agent that can implement a DRL algorithm and gain an optimal strategy of action decision.}
		\label{fig}
	\end{figure}
	\begin{itemize}
		\item \textbf{Observation $o_t^{i}$ :}  $\boldsymbol{o}_{t}^{i}=\left(o_{1}^{i}, \ldots, o_{t}^{i}\right)$, where $ \boldsymbol{o}_{t}^{i}\in \boldsymbol{\mathcal{O}_t}^{i} $. Every agent observes the length of the vehicle queue on entering lanes of their intersection. Moreover, to cater to the design of the proposed \textit{$ \gamma $-Reward} algorithm, we also need to observe the number of vehicle on the exiting lanes which connect to neighbor intersections.
		\item \textbf{Action $ a_t^{i} $:}  $\boldsymbol{a}_{t}^{i}=\left(a_{1}^{i}, \ldots, a_{t}^{i}\right)$, where $ \boldsymbol{a}_{t}^{i}\in \boldsymbol{\mathcal{A}_t}^{i} $. Action can be easily set as the serial number of phase which is chosen to be green.
		\item \textbf{Transition probability $ \mathcal{P} $:} $ \mathcal{P}{(o_{t+1}}^{i}|{o_t}^{i},{a_t}^{i}) $ describes the probability from state $ {o}_t^{i} $ to the next potential state $ {o_{t+1}}^{i} $.
		\item \textbf{Reward $ r_t^{i} $:} After executing each action $ a_t^{i} $, we can get a return information to judge whether $ a_t^{i} $ is good enough for $ o_t^{i} $. We use the number of waiting vehicle on the entering lanes as a raw reward. For amendatory reward, we use $ R_t^{i} $ as a representation.
		\item \textbf{Policy  $ \pi $:} Policy is what agents need to learn in DRL. It represents the goal of reducing travel time and increasing average speed. For a single agent, $\pi^{i}:{\mathcal{O}_t}^{i} \mapsto \mathcal{A}_t^{i}$.
		\item \textbf{Discount rate $ \gamma' $:} This factor is the common meaning used in Temporal-Difference (see Appendix A.1). To avoid confusion with \textit{$ \gamma $-Reward} , we use $\gamma'$ here to replace the original symbol $ \gamma $.
	\end{itemize}
	
	By using the Bellman equation, the relationship among these parameters can be formulated. We can gain the optimal phases after iterating these equations:
	\begin{equation}
	\begin{aligned}
	Q(o_t^i,a_t^i)&=Q(o_t^i,a_t^i)+\alpha (r^i_t+\gamma' \max_{a_{t+1}^i}Q(o_{t+1}^i,a_{t+1}^i)-Q(o_{t}^i,a_{t}^i))\\
	a_t^i&=\arg \max Q(o_t^i,a_t^i)
	\end{aligned}
	\end{equation}
	
	\subsection{Proposed Spatial Differentiation Function}
	Coordination among agents plays a critical role in MARL algorithms, either centralized or decentralized. In this paper, we propose a coordination mechanism among distributed agents. Each agent is based on Dueling-Double-Deep Q Network (D3QN)\cite{mnih2015human}\cite{VanHasselt2016}\cite{Wang2016}, which is one of the best Q value-based model until now\cite{hessel2017rainbow}. A detailed description of it can be found in Appendix A. We found that some studies already use D3QN directly on the TSC problem \cite{liang2019deep}, but it can only be used as an independent Q-learning (IQL) method in the multi-intersection problem. For the TSC problem, the idea of distributed agents is wise and what we need is an appropriate decentralized coordination mechanism. Therefore, we propose \textit{$ \gamma $-Reward} framework for coordination (Figure \ref{interaction}).
	\begin{figure}[htbp]
		\centerline{\includegraphics[width=9cm]{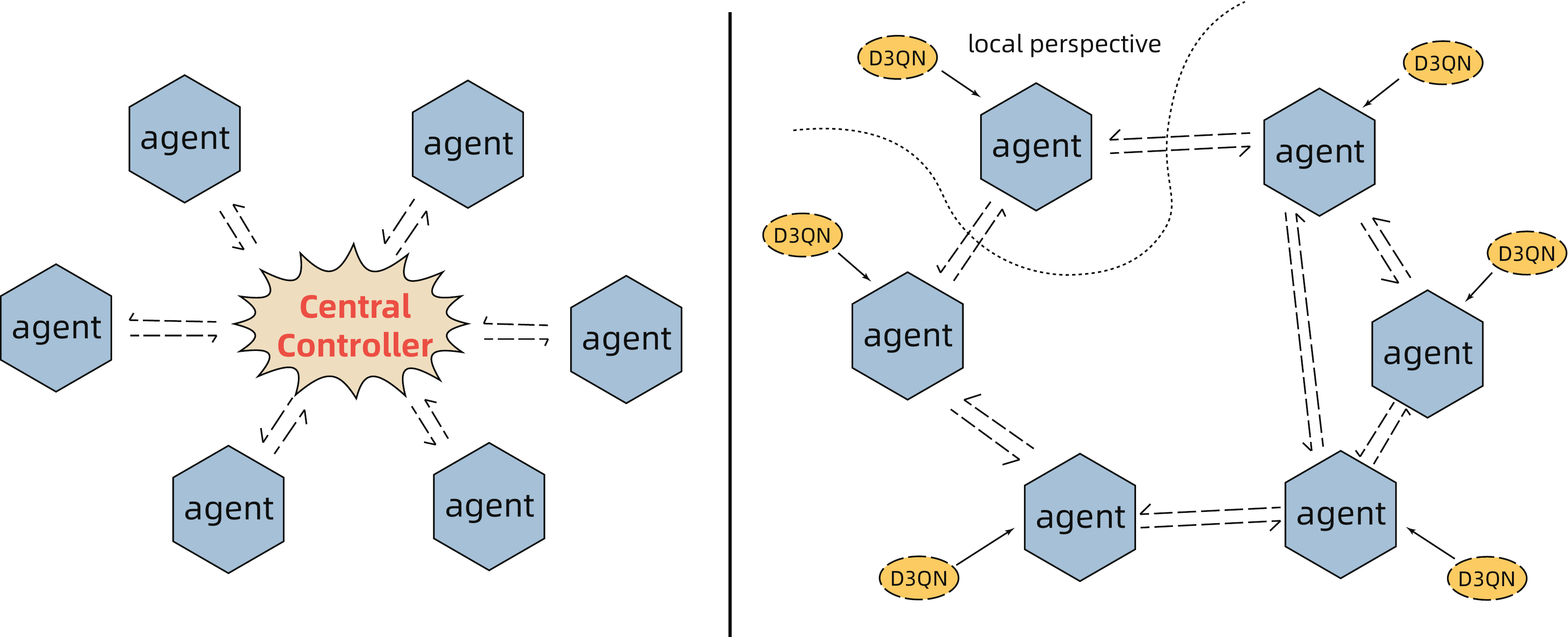}}
		\caption{ Left: MARL algorithms with centralized setting. Right: $\gamma$-Reward framework which is a decentralized one with network agents. $\gamma$-Reward contains a coordination mechanism proposed based on D3QN.}
		\label{interaction}
	\end{figure}
	
	The basic theory of DRL inspires the main idea of this article.  For the TSC problem, not only the temporal decision should be regarded as MDP, the road network itself is more like a Markov Chain since the decision of the current intersection will affect other intersections in the future. So it should consider not only the TD-error (see Appendix A.1) in time, but also a “TD-error” in space which is defined as the \textit{spatial differentiation}.

	\begin{figure}[htbp]
		\centerline{\includegraphics[width=6cm]{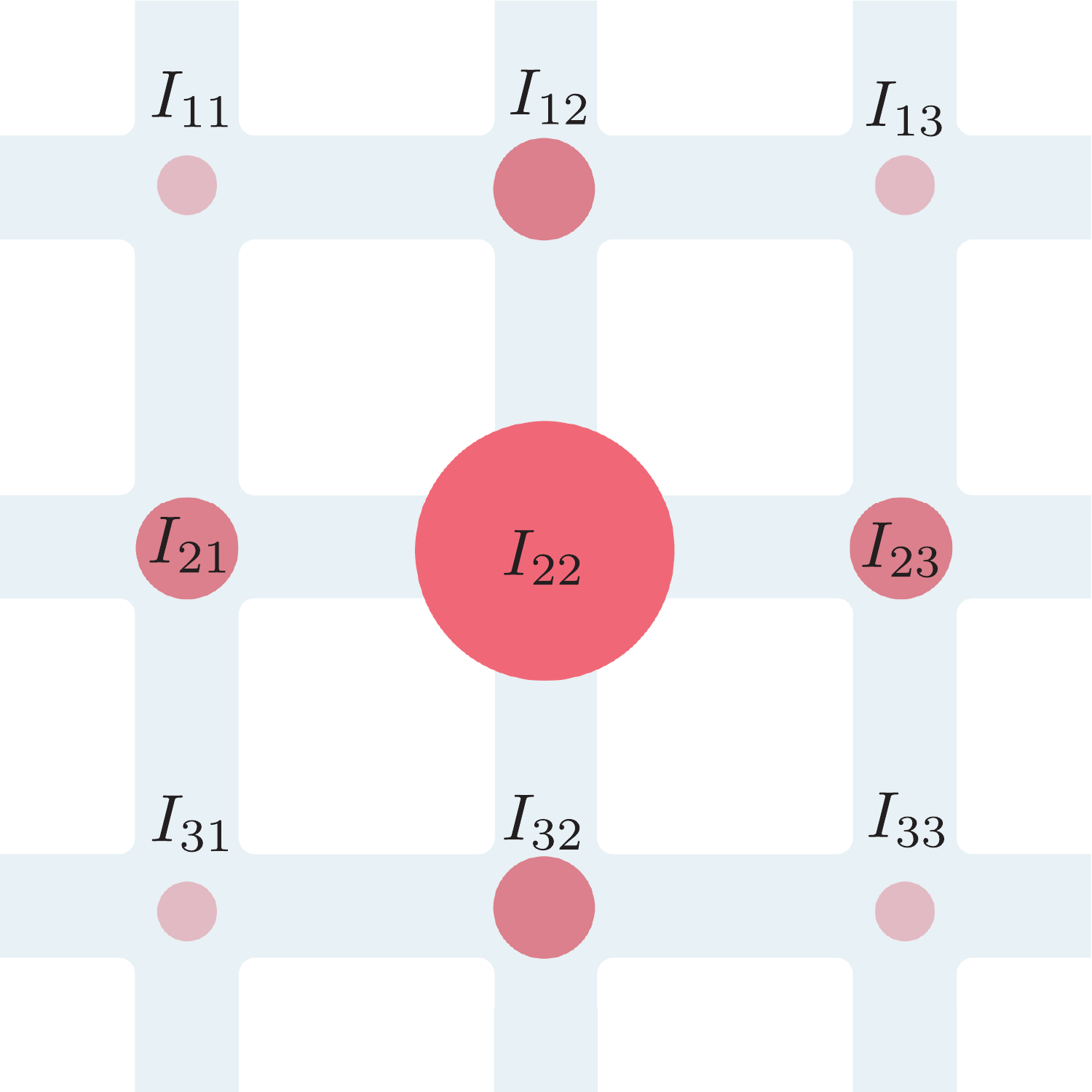}}
		\caption{Diagram for spatial differentiation}
		\label{gamma-reward}
	\end{figure}
	
	Figure \ref{gamma-reward} shows a diagram of multi-intersection road network with $ 3\times 3 $ intersections. Since the decision of intersection $ I_{22} $ at time $ t $ will affect the next intersection $ I_{23} $ at time $ t+n $, the result of intersection $ I_{22} $ at time $ t $ should be affected not only by the current intersection reward $ r^{I_{22}}_t $, but also by the reward $ r^{j}_{t+ n} $ of the surrounding intersections at time $ t+n $, where $ j\in{I_{12}, I_{21}, I_{23}, I_{32}} $. The formula of \textit{spatial differentiation} is as follows:
	\begin{equation}\label{GR}
	\begin{aligned} &R^{I_{22}}_t =r^{I_{22}}_t \cdot
	\left\{1+\gamma \cdot \text {tanh} \left[ \sum_{j \in \mathcal{N}_i}\left(\frac{R^{j}_{t+n}}{r^{j}_t}-c\right)\right]\right\}
	\end{aligned}
	\end{equation}
	
	Where $ \mathcal{N}_{I_{22}}=\{I_{12}, I_{21}, I_{23}, I_{32}\} $ represents the four intersections around intersection $ I_{22} $. The parameter $ n $ is called \textit{delay span}  which represents the time span of most vehicles reach the  next intersection after current action $ a_t $. It is related to the length of the road, the average velocity on the exiting lane and sample interval. The whole formula can also be treated as value function of reward, thus it is a \textit{spatial differentiation} of \textit{reward-value function}.
	
	In Equation \ref{GR}, $ {R^j_{t+n}}/{r^j_t} $ shows the change of traffic capacity at intersection $ j $ between time $ t $ and $ t+n $. If $ {R^j_{t+n}}/{r^j_t} $ is greater than threshold $c$,  it indicates that the traffic capacity of the intersection $ j $  is deteriorated, in other words, the decision of the intersection $ I_{22} $ at the time $ t $ will cause the adjacent intersection $j$ to be more congested; conversely, $ {R^j_{t+n}}/{r^j_t} $ less than $c$ indicates that the capacity of intersection $j$ is improved, and the decision of the intersection $I_{22}$ at time $ t $ is good for $ j $. 
	Through some region transformations, the differential item is served as a penalty. It finally multiplies by a spatial discount factor $ \gamma $ as an amendment to reward. Since the training goal of DRL is to maximize the reward function, the existence of the penalty item forces the agent to pay attention to the situation around the intersection while improving its own policy. 
	
	This equation is also in line with the idea of MARL with networked agents which dates back to Varshavskaya et al. \cite{varshavskaya2009efficient}.  $\mathcal{QD}$-learning\cite{kar2013cal} further gives a convergence proof of this kind of MARL algorithm. It incorporates the idea of \textit{consensus + innovation} to the standard Q-learning algorithm and has the following Q-value update equation:
	
	\begin{equation}\label{QD}
	\begin{aligned} 
	Q(o_t^i,a_t^i) &\leftarrow Q(o_t^i,a_t^i)\\&+\alpha_{t, o, a}\left[r^{i}_t+\gamma' \max _{a_{t+1}^i \in \mathcal{A}} Q(o_{t+1}^i,a_{t+1}^i)-Q(o_{t}^i,a_{t}^i)\right]\\&-\beta_{t, o, a} \sum_{j \in \mathcal{N}_{t}^{i}}\left[Q(o_t^i,a_t^i)-Q(o_t^j,a_t^j)\right]
	\end{aligned}
	\end{equation}
	
	where the term after $\alpha_{t,o,a}$ denotes a local \textit{innovation} potential that incorporates newly obtained observations, and the term after $\beta_{t,o,a}$ is a \textit{consensus} potential (agent collaboration) which captures the difference of Q-value estimates from its neighbors.
	
	Analogously, our Q-value update equation can be changed as follows:
	\begin{equation}
	\begin{aligned}
	Q(o_t^i,a_t^i)=&Q(o_t^i,a_t^i)+\alpha \{ r^i_t+\gamma\cdot r^i_t \cdot \text {tanh} [ \sum_{j \in \mathcal{N}_i}(\frac{R^{j}_{t+n}}{r^{j}_t}-c)]\\&+\gamma' \max _{a_{t+1}^i \in \mathcal{A}}Q(o_{t+1}^i,a_{t+1}^i)-Q(o_{t}^i,a_{t}^i)\}
	\end{aligned}
	\end{equation}
	
	The spatial differentiation term denotes the consensus like $\mathcal{QD}$-learning and we further construct the communication in the temporal domain and the spatial domain at the same time. Although Equation \ref{GR} uses future information which sounds contrary to the causality, it is achievable in the programming. Since we use  Q-learning as a basic model which is off-policy, it saves the trajectory of state-action pair and the obtained reward into a replay buffer. In this way, the raw reward $ r^i_t $ is saved first and then corrected after $n$ steps. Therefore the calculation process can be realized (see Appendix B for pseudocode).

	\subsection{Attention Mechanism for $ \gamma$-Reward}
	\begin{figure*}[htbp]
		\centerline{\includegraphics[width=1\linewidth]{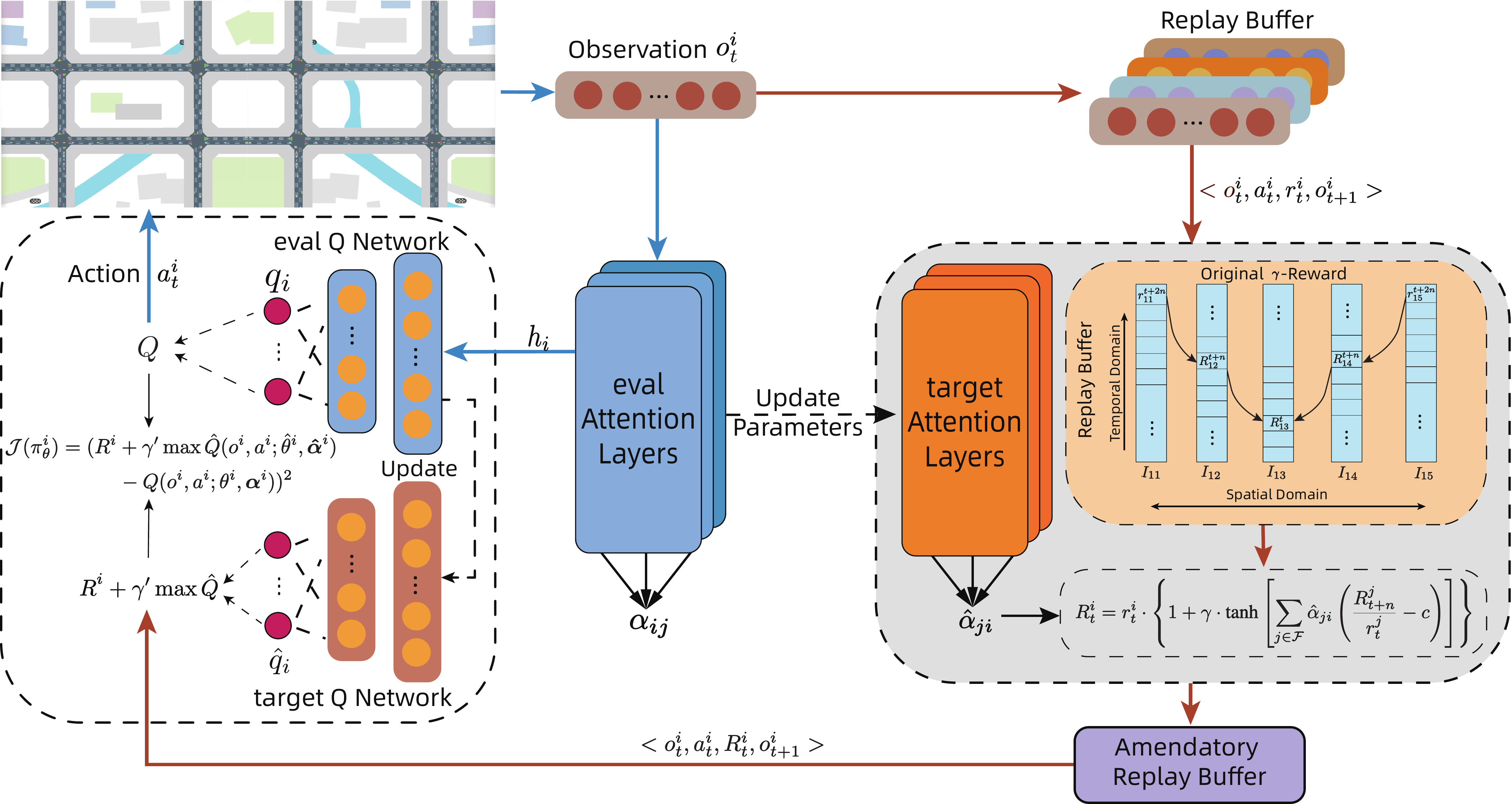}}
		\caption{Framework of the proposed \textit{$ \gamma $-Attention-Reward} model; In the inner cycle, eval Attention Layers and eval Q Network in blue are used to evaluate real-time Q value for the control. In the external cycle, target system in orange is used to predict long-term impact for improving the performance of Q network and updated periodically. At the right of this figure, we illustrate the message mechanism in original \textit{$\gamma $-Reward} by the example of a $Arterial_{1\times5}$ road network.}
		\label{Network}
	\end{figure*}
	The \textit{spatial differentiation} formula proposed in the previous section is based on the fact that each intersection in the road network has the same situation. In other words, the levels of them are equal. But in reality, the levels of the intersections in the road network are different, some have only one or two lanes, and others may have four to six. Imagine if the intersection $ I_{22} $ in Figure \ref{gamma-reward} is a two-way road which has eight lanes, the intersection $ I_{23} $  is same as $ I_{22} $, while on the other side, the intersection  $ I_{32} $  is a two-way which only has two lanes. The decision of intersection $ I_{22} $ in $ I_{22}\Rightarrow I_{23} $ and $ I_{22}\Rightarrow I_{32} $ is inevitably different. The discharge of intersection $ I_{22} $ can easily lead to excessive congestion at intersection $ I_{32} $, but it may not be very serious for intersection $I_{23} $. In addition to the number of lanes at the intersection, the length between each intersection is different, which means that the maximum congestion length each intersection can accept is also different. In summary, when using a \textit{spatial differentiation} formula correction at an intersection, there are different influence weights for different intersections.
	
	The problems mentioned above can indeed be solved by setting different thresholds to get weights, like \cite{liu2017distributed}. However, the actual road network situation is very complicated. There is no way to include all parameters into consideration. Therefore, we can learn the rules from the traffic data. In this respect, the attention mechanism gives us a good solution.
	
	Attention mechanism \cite{Bahdanau2014}\cite{Cho2014}\cite{Vaswani2017}  is an algorithm first proposed to solve “seq2seq” translation problem in NLP. Attention can be interpreted broadly as a vector of importance weights: To predict an element, such as a word in a sentence, attention vector can be used to estimate how strongly it is related to other elements, and the sum of its values can be used as an approximation of the target. It breaks the limits of Euclidean distance between data, captures long-range dependencies in the sentences, and provides smoother translations. 
	
	In addition to sequence data, Attention can also be used for other types of problems. In the graphics world, the GCN \cite{kipf2016semi} tells us that combining local graph structures with node features can achieve good performances in node classification tasks. However, the way GCN combines the characteristics of neighboring nodes is closely related to the structure of the graph, which limits the generalization ability of the trained model on other graph structures. The GAT \cite{velivckovic2017graph} proposes a weighted summation of neighboring node features using the attention mechanism. The weights of the neighboring node feature depend entirely on the node characteristics and are independent of the graph structure.

	Recent research has begun to introduce the idea of Attention to MARL algorithms. A Multi-Agent Actor-Critic (MAAC) algorithm has been proposed that combines attention mechanism\cite{Iqbal2018}. MAAC encodes the state of the surrounding agents and obtains the contribution value of the surrounding agents to the current agent through the attention network, together with $(o, a) $ of the current agent as an input, the Q value is obtained through an MLP. While the Q  network is updated in reverse, the attention network is updated, and the attention scores of the surrounding agents for the current agent are also corrected. Colight applied attention mechanism to the TSC problem of the large-scale road network, it encodes the state and directly obtains the Q value through the Multi-head Attention network.
	
	\subsubsection{Attention}
	Every agent can interact with their environment and get the observation on time. We first need to embed the observation from the environment by applying a layer of Multi-layer Perception (MLP):
	\begin{equation}
	z_{i} =W_l o_{i}+b
	\end{equation}
	
	To get the weight of the intersection $ i $ to the adjacent intersection $ j $, we need to combine their hidden variables $ z_i, z_j $ by following dot product:
	
	\begin{equation}
	e_{ij}=z_j^T W^T_k W_q z_i
	\end{equation}

	$ e_{ij} $, represents the influence of the adjacent intersection $ j $ on the current intersection $ i $. It should be noted that the influence between them may not be necessarily equal. Then we normalize them by softmax function:
	\begin{equation}\alpha_{i j}=\frac{\exp \left(e_{i j}\right)}{\sum_{k \in \mathcal{N}_i} \exp \left(e_{i k}\right)}\end{equation}
	
	Impact value $ v_{ij} $ can be calculate as $ v_{ij}=\alpha_{ij} z_j$, which means the value $ i $ needs to consider from $ j $. Finally, adding them together and passing the RELU activation function, the final characterization of the intersection $ i $ is obtained:
	\begin{equation}h_{i} =\sigma\left(\sum_{j \in \mathcal{N}_i} \alpha_{i j} z_{j}\right)\end{equation}
	
	\subsubsection{$ \gamma $-Attention-Reward}
	We use the attention mechanism to make \textit{spatial differentiation} function more perfect and interpretable by adding an attention score before the sum operation.
	\begin{equation}\label{GAR}
	\begin{aligned} &R^{i}_t =r^{i}_t \cdot
	\left\{1+\gamma \cdot \text {tanh} \left[ \sum_{j \in \mathcal{N}_i}\hat{\alpha}_{ji}\left(\frac{R^{j}_{t+n}}{r^{j}_t}-c\right)\right]\right\}
	\end{aligned}
	\end{equation}
	
	Attention score is updated together with the policies:
	\begin{equation}\label{J}
	\mathcal{J}(\pi_{\theta}^i)=(R^i+\gamma' \max{\hat{Q}(o^i, {a}^i;\hat{\theta}^i, \boldsymbol{\hat{\alpha}}^{i})}-Q(o^i,a^i;\theta^i,\boldsymbol{\alpha}^i))^2
	\end{equation}
	
	It is worth emphasizing that $\alpha_{ij} $ represents the importance from $ j $ to $ i $, but for \textit{$ \gamma $-Reward} we seek the influence from $ i $ to $ j $. So in Equations \ref{GAR} and \ref{J}, we need to use $ \alpha_{ji} $ for amending rewards of agent $ i $.
	
	Since the attention score $\alpha_{ji} $ is a real-time updated value, this paper uses it as an evaluation metric to dynamically assess the impact of surrounding intersections based on dynamic traffic data. While the reward is also an evaluation indicator, which is timely feedback obtained after performing an action at a particular state, used to evaluate the quality of the state-action pair. If we introduce $\alpha_{ij} $ into the calculation of reward-value and update it in real-time, there is bound to be a problem. This causes the Attention layer to update the direction intentionally, which reduces the impact of essential neighbor intersections and increases those with less traffic flow to increase the reward-value. In this way, the introduction of the attention score will be even worse than the original \textit{$\gamma$-Reward}. To solve this problem, we can follow the Q-learning algorithm and use the off-policy idea to get target attention scores $\hat{\alpha}_{ij}$ in the reward calculation using target Attention layers. Its network parameters are updated together with target Q parameters $\hat{\theta}$. The whole framework of the proposed \textit{$ \gamma $-Attention-Reward} model is shown in Figure \ref{Network}.
	
	In Colight, there is a section devoted to the selection of the hyper-parameter neighbor scope. It is found through experiments that the larger the $\left|\mathcal{N}_{i}\right|$ is, the better the performance is. But when it is greater than 5, it takes more time to learn. This is because intersections within the scope of $\left|\mathcal{N}_{i}\right|$  needs to aggregate all the observations into one agent, and the total number of agents is still equal to the intersection, which will inevitably lead to an increase in the amount of calculation. However, unlike Colight, \textit{$ \gamma $-reward} does not need to consider the size setting of the neighbor number. We can merely consider the neighbor number as a constant 5, which contains the current intersection and the four intersections directly connected to it. Note that, we do not need to employ Multi-head Attention which leads to centralization. For the information of further intersections, recursiveness in the \textit{$ \gamma $-Reward} formula can work. 
	
	From the original \textit{$\gamma $-Reward} part in Figure \ref{Network}, we can figure out its principle. The row in the figure represents the state information collected in time series and the column represents each intersection. To prevent too much confusion, only the \textit{$\gamma$-Reward} process of intersection $I_{13}$ is shown as a demonstration, and the interval of the scale variable $ n $ is also not shown here. The value of $R_{13}^t$ is related to $R_{12}^{t+n}$ and $R_{14}^{t+n}$. By analogy, these two rewards are related to $R_{11}^{t+2n}$ and $R_{15}^{t+2n}$ respectively. It is worth noting that only a two-dimensional replay buffer is drawn here, in fact, the real intersection has four or more surrounding intersections, so it should be a three-dimensional gradient.
	
	\section{Theoretical Analysis in Game Theory}
	In this section, we establish theoretical results for the proposed algorithms. The key challenge of MARL algorithms is that the decision process may be non-stationary\cite{omidshafiei2017deep}\cite{laurent2011world}. Since an agent is not aware of actions from other agents and lack of communication, the transition probabilities $ \mathcal{P}(o_{t+1}^{i}|o_{t}^{i}\in \mathcal{O}_t^{i}, a_t^{i}\in \mathcal{A}_t^{i})$ are not stationary and change as the other agents change. So we first need to demonstrate the decision process is stationary by using the proposed algorithms. Specifically, MARL can be regarded as a game model\cite{aragon2020traffic}, and we prove that they can converge to Nash equilibrium. The argument is started with the following definitions.
	\begin{df}\label{df1}
		A decentralized MARL decision process is \textit{stationary} (or homogeneous), \textit{iff}, for each agent $ i $ and all $ p,q\in \mathbb{N} $, $ o^{i} \in \mathcal{O}^{i}$, $\boldsymbol{a}^{i} \in \mathcal{A}^{i}$\cite{omidshafiei2017deep}
		\begin{equation}
		\begin{aligned}
		\sum_{\boldsymbol{a}_{p}^{-i}\in \boldsymbol{A}_{p}^{-i}} \mathcal{P}_i\left(o_{p+1}^{i}| o_{p} ^{i},  \left\langle a_{p}^{i}, \boldsymbol{a}_{p}^{-i}\right\rangle  \right) =\\
		\sum_{\boldsymbol{a}_{q}^{-i}\in \boldsymbol{A}_{q}^{-i}} \mathcal{P}_i\left(o_{q+1}^{i}|o_{q} ^{i},\left\langle a_{q}^{i}, \boldsymbol{a}_{q}^{-i}\right\rangle \right)
		\end{aligned}
		\end{equation}
		and $ \mathcal{P} $ for global state $ s\in \mathcal{S} $ must be stationary either.
		\begin{equation}
		\begin{aligned}
		\sum_{\boldsymbol{a}_{p}^{-i}\in \boldsymbol{A}_{p}^{-i}} \mathcal{P}\left(s_{p+1}^{i}| s_{p} ^{i},  \left\langle a_{p}^{i}, \boldsymbol{a}_{p}^{-i}\right\rangle  \right) =\\
		\sum_{\boldsymbol{a}_{q}^{-i}\in \boldsymbol{A}_{q}^{-i}} \mathcal{P}\left(s_{q+1}^{i}|s_{q} ^{i},\left\langle a_{q}^{i}, \boldsymbol{a}_{q}^{-i}\right\rangle \right)
		\end{aligned}
		\end{equation}
	\end{df} 
	where $ \boldsymbol{a}^{-i}=\boldsymbol{a} \backslash\left\{a^{i}\right\}$.
	
	Based on the definition of stationary MDP, we can define optimal global reward for proving the astringency of proposed methods by extending the definition by \cite{nguyen2014decentralized}. 
	\begin{df}\label{df2}
		For a stationary MDP, the global optimal reward can be decomposed into a sum of local optimal reward for each reward function $f_{i} \in \mathcal{F}$
		\begin{equation}
		\rho^{*}=\sum_{i=1}^{m} \rho_{i}^{*}
		\end{equation}
	\end{df} 
	\begin{proof}
		For a given stationary MDP, there must exists a stationary $\mathcal{P}_i^{\pi_*^i}$, where $ \pi_*^i $ is the optimal policy of agent $ i $
		\begin{equation}
		\rho_i^{\pi_*^i}=\sum_{o^{i}\in \mathcal{O}^{i}}\mathcal{P}_i^{\pi_*^i}(o^{i})f_i(o^{i}, a^{i}| a^{i}=\pi_*^i(o^{i}))
		\end{equation}
	\end{proof}
	
	\begin{df}\label{df3}
		In stochastic game, a Nash equilibrium point is a tuple of $ m $ strategies $ \left(\pi_{*}^{1}, \ldots, \pi_{*}^{m}\right) $
		such that for all global state $ s\in \mathcal{S} $ and $ i = 1,\dots,m $ \cite{hu2003nash}
		\begin{equation}
		\nu^i(s|\pi_{*}^{1}, \ldots, \pi_{*}^{m})\leq \nu^i(s|\pi_{*}^{1}, \ldots, \pi^{i}, \pi_*^{i+1}, \ldots, \pi_{*}^{m})
		\end{equation}
		for all $ \pi_{*}^i\in \Pi^i  $, where $ \Pi^i $ is the set of policies of total $ m $ agents.
	\end{df}
	
	\subsection{Stationarity of \textit{$\gamma $-Reward} series}
	First, we give the proof of stationarity. Unless the MDP is stationary, the model cannot guarantee convergence to the optimal.
	\begin{as}\label{as1}
		The original reward function can be represented as
		\begin{equation}
		r_t^{i}=f(\boldsymbol{o}_t^{i}, \boldsymbol{a}_t^{i})
		\end{equation}
		where $ \boldsymbol{o}^{i}, \boldsymbol{a}^{i} $ include state-action pair from time step $ 1 $ to $ t $.
		Assume that \textit{$ \gamma $-Reward} series are special reward function $ f(o,a) $.
		\begin{equation}\label{R}
		R_{t}^{i}=f\left(\left\langle \boldsymbol{o}_{t}^{i}, \boldsymbol{o}_{k}^{-i}\right\rangle ,\left\langle \boldsymbol{a}_{t}^{i}, \boldsymbol{a}_{k}^{-i}\right\rangle \right)
		\end{equation}
		$ k\in [1,\mathbb{N}] $.
	\end{as}
	\begin{as}\label{as2}
		As a continuing task without definite ending, the excepted reward $ \mathcal{G}_t^{i} $ of TSC problem is defined as following with a discount rate $ \gamma' $
		\begin{equation}
		\mathcal{G}_t \doteq r_{t+1}+\gamma' r_{t+2}+(\gamma')^{2} r_{t+3}+\cdots=\sum_{k=0}^{\mathbb{N}-t} (\gamma')^{k} r_{t+k+1}
		\end{equation}
	\end{as}
	\begin{as}\label{as3}
		The Q function is based on trajectory of expected return.
		\begin{equation}
		Q(o,a|\pi)=\sum\mathcal{P}(path_{o,a}|\pi)*\mathcal{G}(path_{o,a})
		\end{equation}
		\begin{equation}\label{argmax}
		a_t=\text{\textit{argmax} }  Q_{\max}(o,a| \pi)
		\end{equation}
	\end{as}
	
	\begin{theorem}\label{th1}
		With \textit{$\gamma $-Reward} as a coordination mechanism, the decision process of distributed DQN algorithm $ \mathcal{(RP)} $ is stationary.
	\end{theorem}
	\begin{proof}
		According to Assumption \ref{as2} and Assumption \ref{as3}, $ Q $ value function with  \textit{$\gamma $-Reward} $ \mathcal{(RQ)} $ can be written as follow
		\begin{equation}
		\begin{aligned}
		\mathcal{(RQ)}_t^{i}(o,a|\pi)&=\sum_{t=0}^{\mathbb{N}}\mathcal{(RP)}_t^i(path_{o^{i},a^{i}}|\pi)*(\mathcal{RG})_i(path_{o^{i},a^{i}})\\
		&=Q\left(\left\langle \boldsymbol{o}_{t}^{i}, \boldsymbol{o}_{k}^{-i}\right\rangle ,\left\langle \boldsymbol{a}_{t}^{i}, \boldsymbol{a}_{k}^{-i}\right\rangle \lvert  \pi \right)
		\end{aligned}
		\end{equation}
		
		The calculation of $ \mathcal{(RQ)} $ is related to the except reward path. $ \mathcal{(RG)}_t^i $ is a discounted sum of amendatory reward $ R_t^{i} $, which is bound up with not only $ (o^{i},a^{i}) $, but also $ (o^{-i},a^{-i}) $ from the other agents. Since $ \mathcal{(RG)}_t^i $ records the future trajectory of amendatory reward from time step $ t $ to the end of episode, it must contain the previous and posterior state-action pair as a vector, like Equation \ref{R} shown in Assumption \ref{as1}. According to the above two properties, we can decompose $ \mathcal{(RP)} $ by using Equation \ref{argmax}.
		\begin{equation}
		\begin{aligned}
		\mathcal{(RP)}_t^i(o_{t+1}^i\lvert {o}_t^i,{a}_t^i)&=\mathcal{P}_t^i(o_{t+1}^i\lvert o_t^i, \text{\textit{argmax} }  \mathcal{(RQ)}^i_{t\max}(\boldsymbol{o},\boldsymbol{a}))\\
		&=\mathcal{P}_t^i\left(o_{t+1}^i\lvert o_t^i, \left\langle a_t^i, \boldsymbol{a}_k^{-i}\right\rangle \right)
		\end{aligned}
		\end{equation}
		where $ k\in [1,\mathbb{N}] $. Obviously, $ \mathcal{(RP)} $ satisfy the property in Definition \ref{df1}.
		Thus, the process with \textit{$ \gamma $-Reward} series is a stationary process.
		
	\end{proof}

	\subsection{Convergence of \textit{$\gamma $-Reward} series}
	\begin{theorem}
		\textit{Spatial differentiation} formula in original \textit{$\gamma$-Reward} can lead the \textit{reward value function} to converge to local optimal reward.
	\end{theorem}
	\begin{proof} 
		As mentioned in Section II.B, our Q-value function shares a similar structure with $\mathcal{QD}$-learning. So the proof is inspired by the convergence analysis of $\mathcal{QD}$-learning. Since this article is not mainly about multi-agent theory, we would like to simply transform $\gamma$-Reward into the form of $\mathcal{QD}$-learning convergence proof. By following the rest of the proof process in $\mathcal{QD}$-learning, we can get the conclusion of convergence. Thus, we highly recommend readers to read the whole analysis in $\mathcal{QD}$-learning.  
			
		According to the \textit{Spectral Graph Theory}, the multi-agent communication network is simplified to an undirected graph $G=(V,E)$, where $V$ represents intersections and $E$ denotes the communication links. From the adjacency matrix $A$ ($A_{ij}$=1, if $(i,j)\in E$, $A_{ij}=0$, otherwise) and the degree matrix $D=\text{diag}\{d_1, \dots, d_N\}$ ($d_i = |\mathcal{N}_i|$), we can get a positive definite graph Laplacian matrix $L=D-A$, where eigenvalues ordered as $0=\lambda_1(L)\leq \lambda_2(L)\leq \cdots \leq \lambda_N(L)$. Let $\beta=-\alpha\gamma$ and $\beta \text {tanh} [ \sum_{j \in \mathcal{N}_i}(\frac{R^{j}_{t+n}}{r^{j}_t}-c)] = -\beta \text {tanh} [ \sum_{j \in \mathcal{N}_i}(c-\frac{R^{j}_{t+n}}{r^{j}_t})]$, we note that
		
		\begin{equation}\label{QDGR1}
		\begin{aligned}
			\boldsymbol{Q}_{\boldsymbol{o},\boldsymbol{a}}(t+1)&=(I_N-\beta L_t + \alpha I_N)\boldsymbol{r}_t \\&+\alpha (\mathcal{C}_{\boldsymbol{o},\boldsymbol{a}}(\boldsymbol{Q_t}) -\boldsymbol{Q}_{\boldsymbol{o},\boldsymbol{a}}(t) +\boldsymbol{\nu}_t)
		\end{aligned}
		\end{equation}
		
		where $\boldsymbol{Q}_{\boldsymbol{o},\boldsymbol{a}} = [Q_{o,a}^1(t),\dots,Q_{o,a}^N(t)]^T$ and $\mathcal{C}_{\boldsymbol{o}, \boldsymbol{a}}(\boldsymbol{Q_t})=[\mathcal{C}^1_{o,a}(Q_t^1), \dots, \mathcal{C}^N_{o,a}(Q_t^N)]^T $. The local $\gamma$-Reward operator $\mathcal{C}^i$ of agent $i$ is
		
		\begin{equation}\label{QDGR2}
		\begin{aligned}
		\mathcal{C}_{o, a}^{i}(Q)=\mathbb{E}(r^i_t)+\gamma' \sum_{j \in \mathcal{N}} p_{i, j}^{a} \max _{v \in \mathcal{A}} Q_{j, v}
		\end{aligned}
		\end{equation}
		
		The residual $\nu^i_t$ for each agent $i$ in Equation \ref{QDGR1} is
		
		\begin{equation}\label{QDGR3}
		\begin{aligned}
        \nu_t^i(o, a) = (1-\dfrac{1}{\alpha})r^i_t  +\gamma' \max _{v \in \mathcal{A}} Q_{{o}_{t+1}, v}^{i}(t)-\mathcal{C}^i_{o,a}\left({Q}\right) 
		\end{aligned}
		\end{equation}
		
		By following the proof in Section V of $\mathcal{QD}$-learning, we can get the same main result that for each agent $i$, we have
		
		\begin{equation}\label{QDGR4}
		\begin{aligned}
		\mathbb{P}(\lim_{t\rightarrow \infty}\boldsymbol{Q}^i_t=\boldsymbol{Q}^*)=1
		\end{aligned}
		\end{equation}
		
		We can conclude that \textit{spatial differentiation} formula can lead the \textit{$\gamma$-Reward} to  converge to local optimal. Obviously,  \textit{$\gamma $-Attention-Reward} has the same property.
		
	\end{proof}

	\subsection{Optimality of \textit{$\gamma $-Reward} model}
	For a stochastic game with multi-agent, the optimal point of the whole system is a Nash equilibrium point which is declared in Definition \ref{df3}. $ \nu^i(s|\pi_{*}^{1}, \ldots, \pi_{*}^{m}) $ can be interpreted as the discounted except reward $ \mathcal{G} $. According to previous two theorems, we can draw the following conclusions.
	\begin{theorem}
		With \textit{$\gamma $-Reward} model, the multi-agent system can converge to a Nash equilibrium point.
	\end{theorem}
	\begin{proof} 
		First we give the definition of optimal $ \nu^* $ with Definition \ref{df2}
		\begin{equation}
		\nu^i(s|\pi_{*}^{1}, \ldots, \pi_{*}^{m})=\rho^{*}=\sum_{i=1}^{m} \rho_{i}^{*}
		\end{equation}
		
		From the process of convergence proof, we can figure that the local optimal reward of $ R^i $ depends on the local optimal of the other agents $ \rho_j^* $. In other words, if there is an agent which does not converge to local optimal reward, the other will also not be optimal. From this, we can conclude that  \textit{$\gamma $-Reward} forces agents to care about the others and let the whole system finally converge to a Nash equilibrium point.
	\end{proof}
	
	\section{Experiment}
	We conduct experiments on Cityflow, an open-source traffic simulator that supports large-scale traffic signal control\cite{Zhang2019CityFlow}, rather than the common used SUMO simulator\cite{Krajzewicz2010Traffic}, since it is more than twenty times faster than SUMO. Moreover, we use Ray\cite{Moritz2018} framework, which is an open-source library for reinforcement learning that offers both high scalability and a unified API for a variety of applications for DRL algorithms.
	
	\subsection{Datasets}
	In the experiment, we use both synthetic data and real-world data. We share the same real-world dataset with Colight for convenience. The datasets mainly include two parts, roadnet and flow. Roadnet describes the number of intersections in the road network, the coordinates, and the number of lanes owned by each intersection. Flow is based on vehicles and lists thousands of vehicles, each vehicle has its own property, such as length, width, max of accuracy, max of speed and, most importantly, trajectory. The experiment used the real-world data of Hangzhou, Jinan in China, and also New York in the USA. Meanwhile, we used two kinds of synthetic data, arterial and grid type. 
	We counted their characteristics and presented them in Table \uppercase\expandafter{\romannumeral1}. $Grid_{3\times 3}uni$ is one-way traffic, and $Grid_{3\times 3 }bi$ is two-way with the same road network.
	\begin{table} 
		\centering
		\caption{Situation of Datasets} 
		\begin{threeparttable}
			\begin{tabular}{cccccc}
				\toprule
				\multirow{2}{2cm}{DataSet}
				& \multirow{2}{2cm}{Intersections}&  \multicolumn{4}{c}{Arrival rate (vehicles/300s)}\\ &&Mean& Std& Max& Min\\
				\midrule
				$ Arterial_{1\times 6}\tnote{$ \dagger $} $ & 6& 300& -&  -&  -\\
				$ Grid_{3\times 3}uni\tnote{$ \dagger $} $ & 9& 300& -&  -&  -\\
				$ Grid_{3\times 3}bi\tnote{$ \dagger $} $ & 9& 300& -& -&-\\
				$ NewYork_{16\times 3} $ & 48& 240.79& 10.08& 274& 216 \\
				$ Jinan_{3\times 4} $ & 12& 526.63& 86.70& 676&  256\\
				$ Hangzhou_{4\times 4} $& 16& 250.79& 38.21& 335& 208 \\
				\bottomrule[1pt]
			\end{tabular}
			\begin{tablenotes}
				\footnotesize
				\item[$ \dagger $] Traffic flow from synthetic data are uniform, so there is no need to count another three values.
			\end{tablenotes}
		\end{threeparttable}
	\end{table}
	\begin{figure*}[htbp]
		\centerline{\includegraphics[width=1\linewidth]{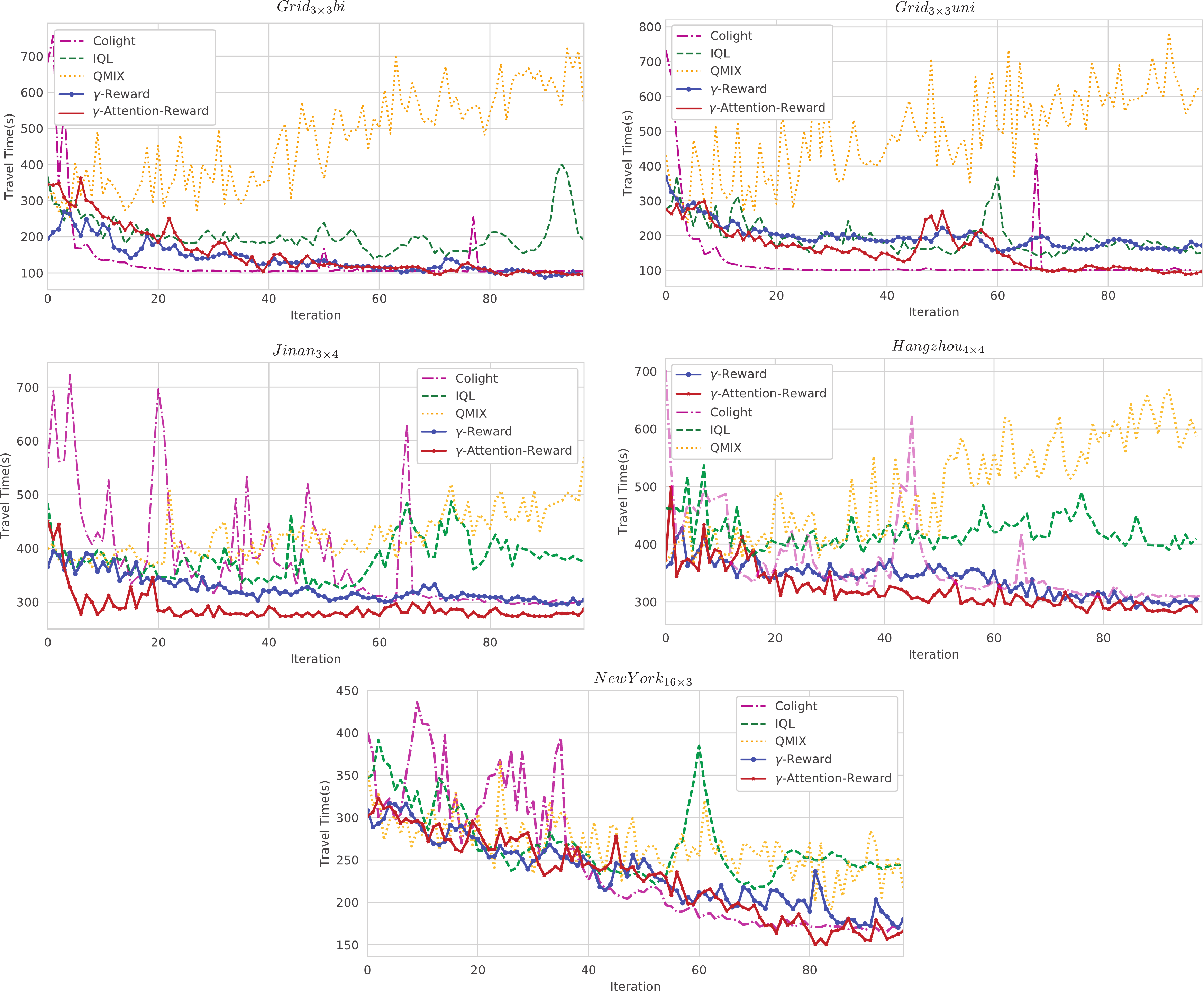}}
		\caption{Evaluation Performance Compared among Baselines, proposed \textit{$ \gamma $-Reward} and \textit{$ \gamma $-Attention-Reward}}
		\label{final_results}
	\end{figure*}
	\subsection{Baseline Methods}
	In Chapter \uppercase\expandafter{\romannumeral2}, we have already introduced methods for TSC, including traditional rule-based methods and learning-based methods. The most primitive rule-based methods are still the most common methods nowadays. 
	As a mature rule-based method, Max-Pressure can be used as a representative.
	
	Learning-based methods have been prosperous under the development of deep learning and data science in recent years. They are characterized by the use of large-scale data to approximate optimal strategies through iterative learning. 
	We have chosen several methods as the baseline:
	\begin{itemize}
		\item \textit{IQL}: Since \textit{$ \gamma $-Reward} is based on the decoupling idea of IQL, and introduces a coordination mechanism so that it can demonstrate the impact of coordination mechanism compared to IQL. Here we use original D3QN, like \cite{liang2019deep}, for comparison.
		\item \textit{QMIX}: This is a sophisticated MARL algorithm, which integrates all agents into the same model and concentrates on the learning of joint action reward functions. 
		As a typical one-model method, comparing it with \textit{$ \gamma $-Reward} can effectively observe the advantages and disadvantages of joint learning and independent learning for coordination.
		\item \textit{Colight}: Unlike \textit{$ \gamma $-Reward}, Colight is more like QMIX, but not learns a joint action. It uses Attention layers to train the surrounding observation code to replace the observation of the current intersection. Due to its full collection of observation, it can apply Multi-head Attention. By this way, the coordination between agents is realized. 
		\textit{$ \gamma $-Attention-Reward} has made some improvements on this basis. The method of Replay Buffer Amendment is employed to introduce the effect of the current intersection on the surrounding intersection, replacing the hyperparameter $\left|\mathcal{N}_{i}\right|$ in Colight and adding consideration of the impact of actions on both time and space.
	\end{itemize}
	
	\subsection{Evaluation Performance}
	Figure \ref{final_results} and Table \uppercase\expandafter{\romannumeral2} shows the performance comparison between \textit{$ \gamma $-Reward} and the more comprehensive \textit{$ \gamma $-Attention-Reward} and baselines. 
	Each model has trained 100 iterations, and each iteration run 3600 time steps in the simulator. Each action in them lasts at least 10 seconds for avoiding rapid switching phase impracticably. Delay span $n=10s$ and threshold $c=0.8$.
	
	We use the average transit time of the vehicle to evaluate the performance of the model, which is the standard evaluation method in the field of TSC.
	
	The performance of learning-based methods is significantly better than rule-based  Max-Pressure (Table \uppercase\expandafter{\romannumeral2}), which is widely proved in many researches.
	
	Among the independent learning DRL model, the performance of \textit{$ \gamma $-Reward}  series far exceeds the IQL model. This demonstrates the importance of coordination between
	agents for global performance improvement. 
	
	It is worth noting that, in all road network, the independent learning DRL model shows a stronger astringency than one-model QMIX. That is probably because for a single model, excessive dimensions can make the policy more difficult to learn. However, Colight does not show divergence while achieved excellent results. This may benefit from that it does not make joint decisions through the joint action function, but by sharing network parameters, so that all agents generate independent actions. Since the agents share the model parameters, they will undoubtedly ignore individual differences and sacrifice some performance. Compared to the proposed model, intense oscillations sometimes occur in Colight during training. This is also a result of sharing parameters. Once the model iterating in the wrong direction, it will mislead all agents and lead to horrible congestion in the whole road network. The performance of proposed methods is even better than Colight, which means sharing sensation is not the only way to realize coordination. Sharing results can also help to focus on the whole road network for a single intersection.
	
	In real-world road networks, \textit{$ \gamma $-Reward} and \textit{$ \gamma $-Attention-Reward} do not show gigantic difference. The reason is that all real-world road networks we used are two-way road. We will introduce the study about the attention score later and can be shown in Figure \ref{hzas} and Figure \ref{attention_score}. It has shown its effect in specifically synthetic road networks. Compared $ Grid_{6\times 6}bi $ and $ Grid_{6\times 6}uni $ in Figure \ref{final_results}, Attention layers distinguish one-way and two-way, and assist agents to achieve a better performance. This will also describe later by revealing the detail of Attention layers.
	
	\begin{table*} 
		\centering
		\caption{Performance Comparison (Average Travel Time)}  
		\begin{tabular}{cccccc}
			\toprule
			Model& $ Grid_{3\times 3}bi $& $ Grid_{3\times 3}uni $& $ NewYork_{16\times 3} $& $ Jinan_{3\times 4} $& $ Hangzhou_{4\times 4} $\\
			\midrule[1pt]
			Max-Pressure & 204.72& 186.06& 405.69& 359.81& 431.53 \\
			\midrule[1pt]
			IQL& 191.05& 157.51& 248.46& 371.74&406.27  \\
			QMIX& 565.70& 619.32& 216.56& 571.78& 587.46 \\
			Colight& 104.89& 100.96&169.66 & 301.78& 311.15 \\
			\midrule[1pt]
			$ \gamma $-Reward& 96.44 & 175.38& 162.18& 303.97& 304.90  \\
			$ \gamma $-Attention-Reward & 96.14& 93.93& 141.16& 286.27& 284.24  \\
			\bottomrule[1pt]
		\end{tabular}
	\end{table*}
	
	\subsection{Study of hyperparameter $\gamma $}
	
	\begin{figure}[htbp]
		\centerline{\includegraphics[width=1\linewidth]{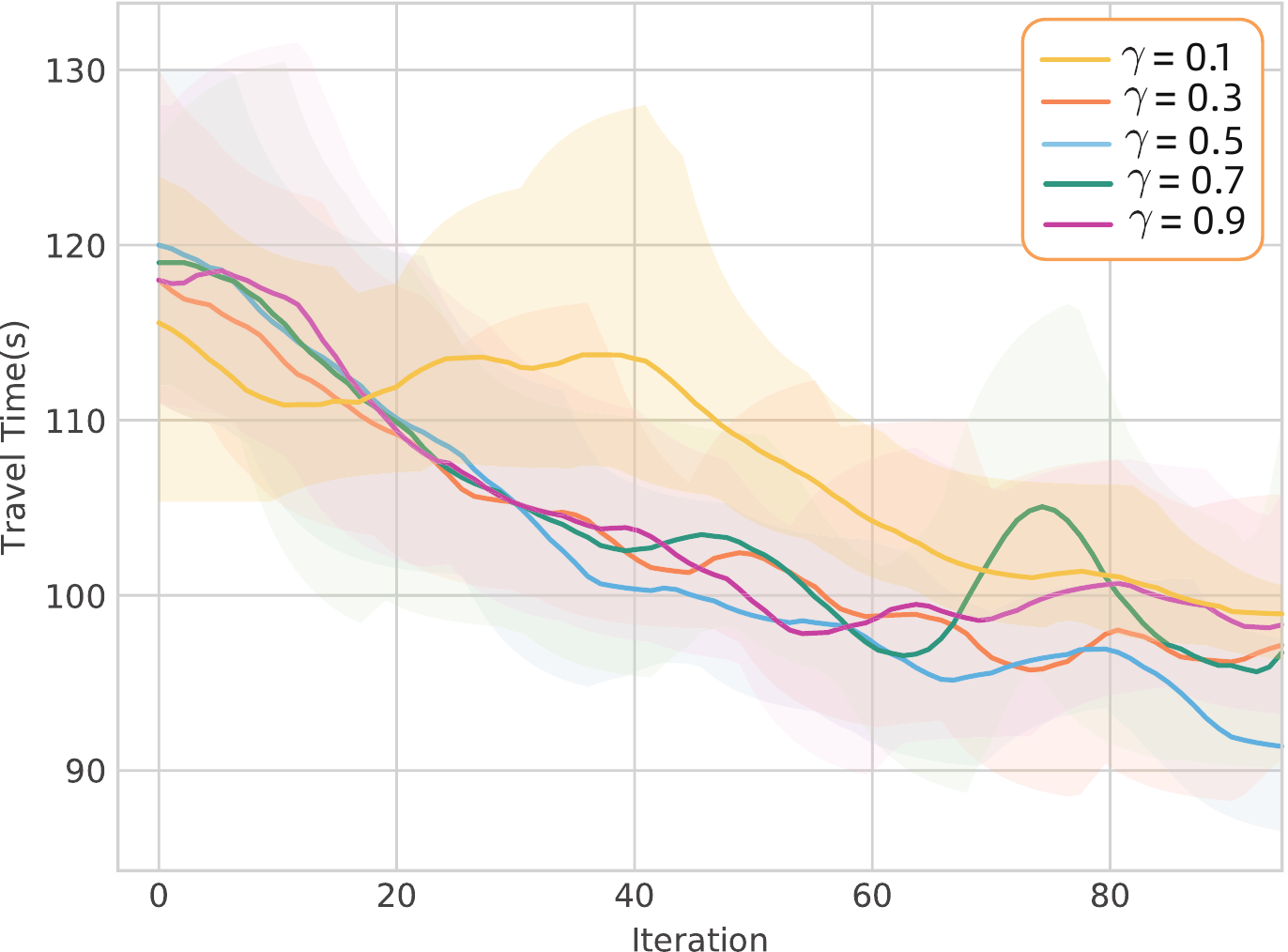}}
		\caption{Study of hyperparameter $\gamma $; dark lines represent results after smoothing, light lines represent the deviation of raw results.}
		\label{gamma_study}
	\end{figure}
	We use  $Arterial_{1\times 6}$ road network to study the impact of different $\gamma$ value. 
	We have chosen [0.1, 0.3, 0.5, 0.7, 0.9] five $ \gamma$ values and compared the results. 
	As shown in Figure \ref{gamma_study}, 0.5 may be a balance point of the penalty item.
	So for the hyperparameter $\gamma$, we all set it to 0.5.
	
	\subsection{Visualization of Proposed method}
	The core idea of the \textit{$ \gamma $-Reward} algorithm is to correct the reward in the replay buffer. 
	In this section, we use the $Arterial_{1\times 6}$ road network as an example to show how the reward values between different intersections affect each other. 
	Compare the $Grid_{3\times 3}uni$ and $Grid_{3\times 3 }bi$ road networks to demonstrate the role of the attention mechanism in the \textit{$\gamma$-Reward} algorithm improvement.
	\begin{figure}[htbp]
		\centerline{\includegraphics[width=1\linewidth]{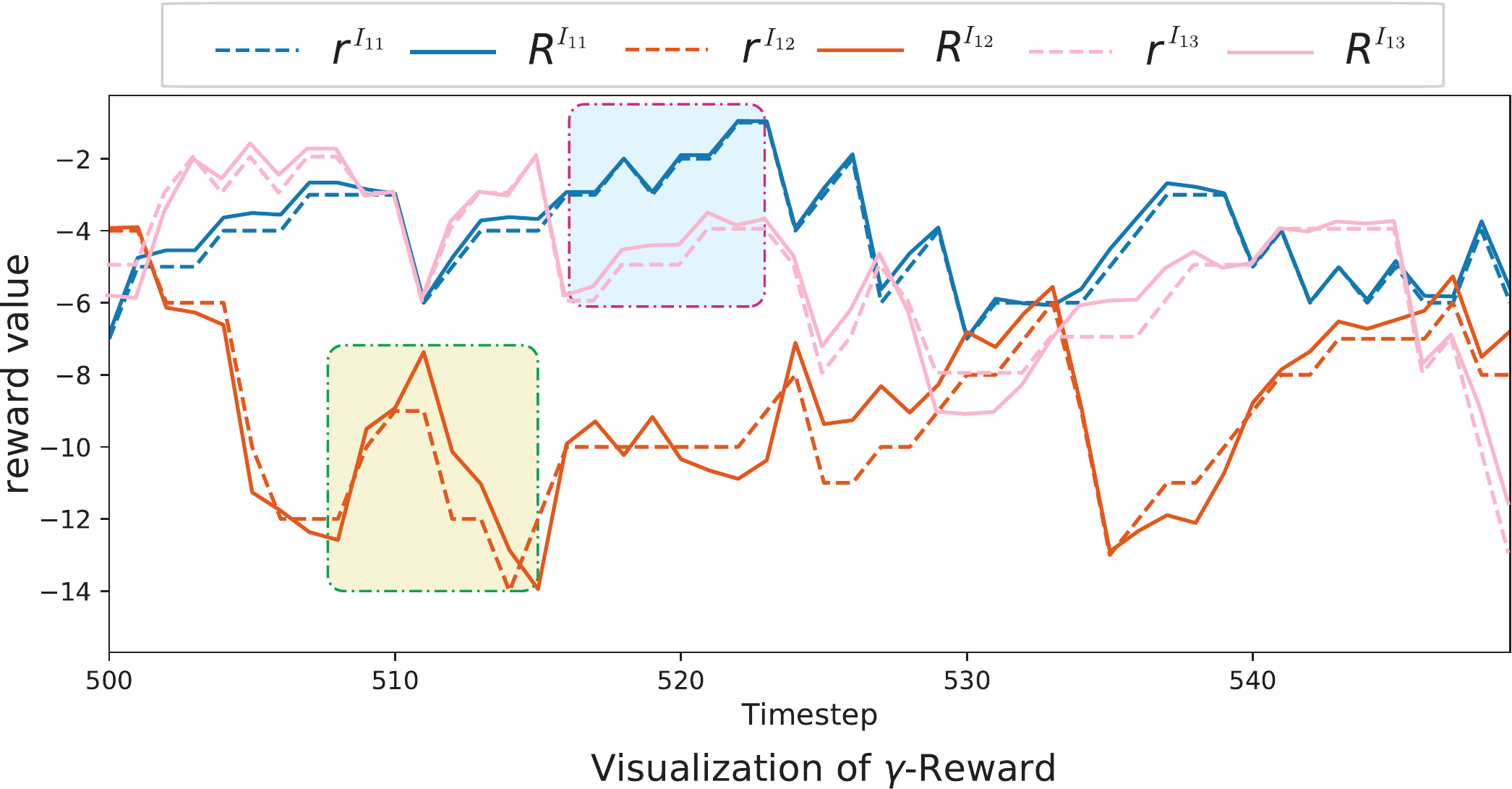}}
		\caption{Visualization of the reward changing by \textit{spatial differentiation}; $ r^i $ represents raw reward from D3QN, $ R^i $ represents the reward after amendment.}
		\label{reward_change}
	\end{figure}
	\subsubsection{Visualization of Spatial Differentiation Formula}
	In Figure \ref{reward_change}, dashed lines represent the original reward, and solid lines represent the corrected reward. 
	Since the linear road network is selected, there are no more than two adjacent intersections, so it is easier to observe the influence between the intersections. We just select the first two intersections for clearness.
	
	Observe the red line in the light yellow area, which represents the reward for the second intersection. It can be found that the solid line in this area is almost above the dashed line, meaning that after amendment, rewards become better. The reason is shown in the light blue area, dark blue and pink lines are getting a raise, which means the traffic situation is getting better both in intersection $I_{11}$ and $I_{13}$. We believe that in the process of getting better at these two intersections, some of the efforts are contributed by intersection $I_{12}$. The lag between yellow and blue area is up to the delay span $ n $.
	
	From Figure \ref{reward_change}, we can see that in the actual training, the \textit{$\gamma$-Reward} framework, as what we expected, introduces future changes in nearby intersections.
	\subsubsection{Effect of Attention Mechanism}
	\begin{figure}[htbp]
		\centerline{\includegraphics[width=0.8\linewidth]{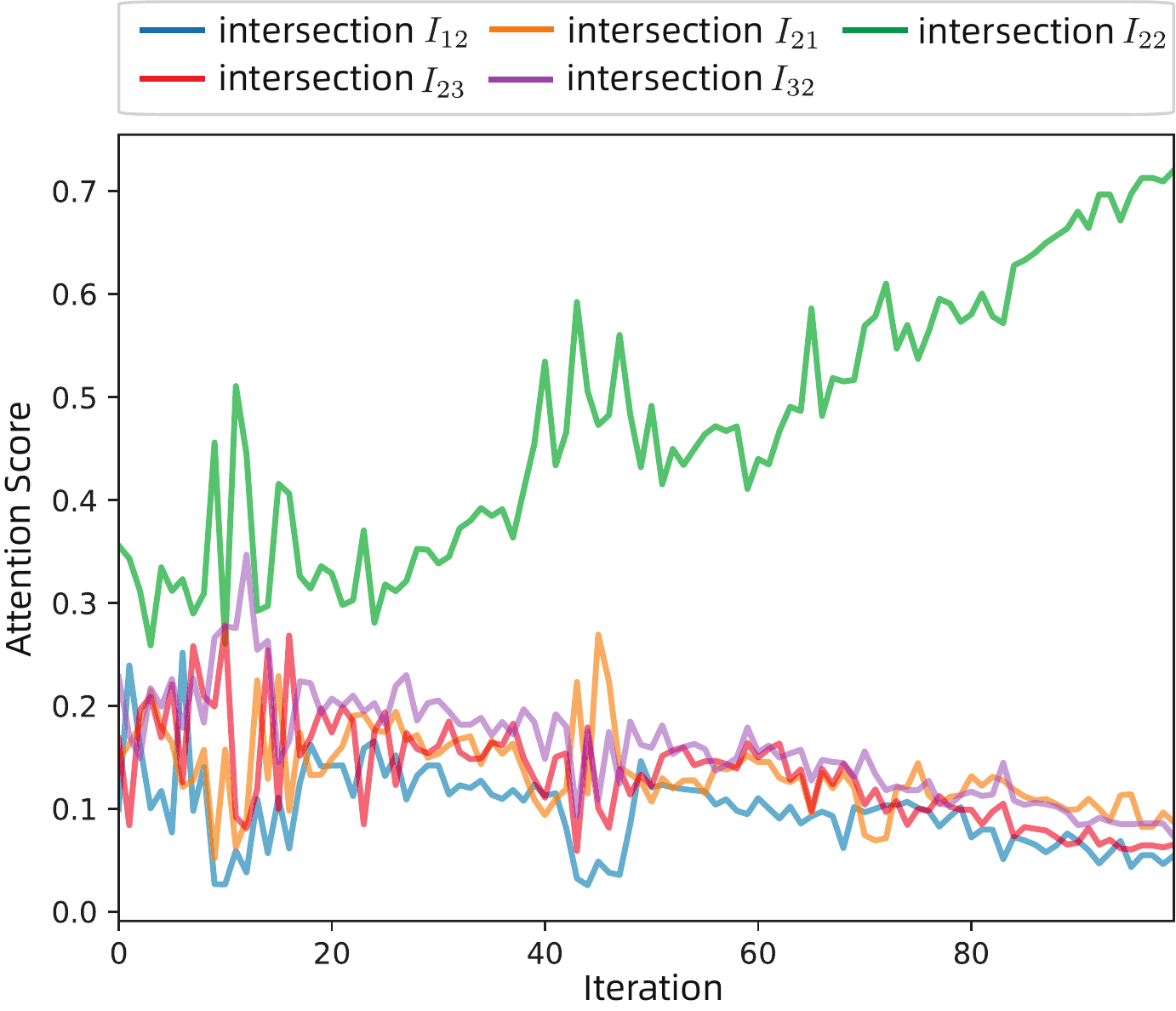}}
		\caption{Attention score of intersection $I_{22}$ selected from $ Hangzhou_{4\times 4} $ road network}
		\label{hzas}
	\end{figure}
	
	Figure \ref{hzas} shows attention scores from Hangzhou road network. We can find that except current intersection $ I_{22} $, others are declined and tending to the same value. Therefore, Attention does not play an important role in these real-world datasets. The reason could be that all of them are two-way road and have the same number of lanes. To highlight its effect, we need to compare the situation between one-way and two-way traffic in the same road network. Thus, we use $Grid_{3\times 3}uni$ and $Grid_{3\times 3 }bi$ for visualization.
	
	\begin{figure}[htbp]
		\centering
		\subfigure[Attention scores of intersection $I_{22}$ selected from $ Grid_{3\times 3}bi $]{
			\begin{minipage}[b]{0.8\linewidth}
				\centerline{\includegraphics[width=1\linewidth]{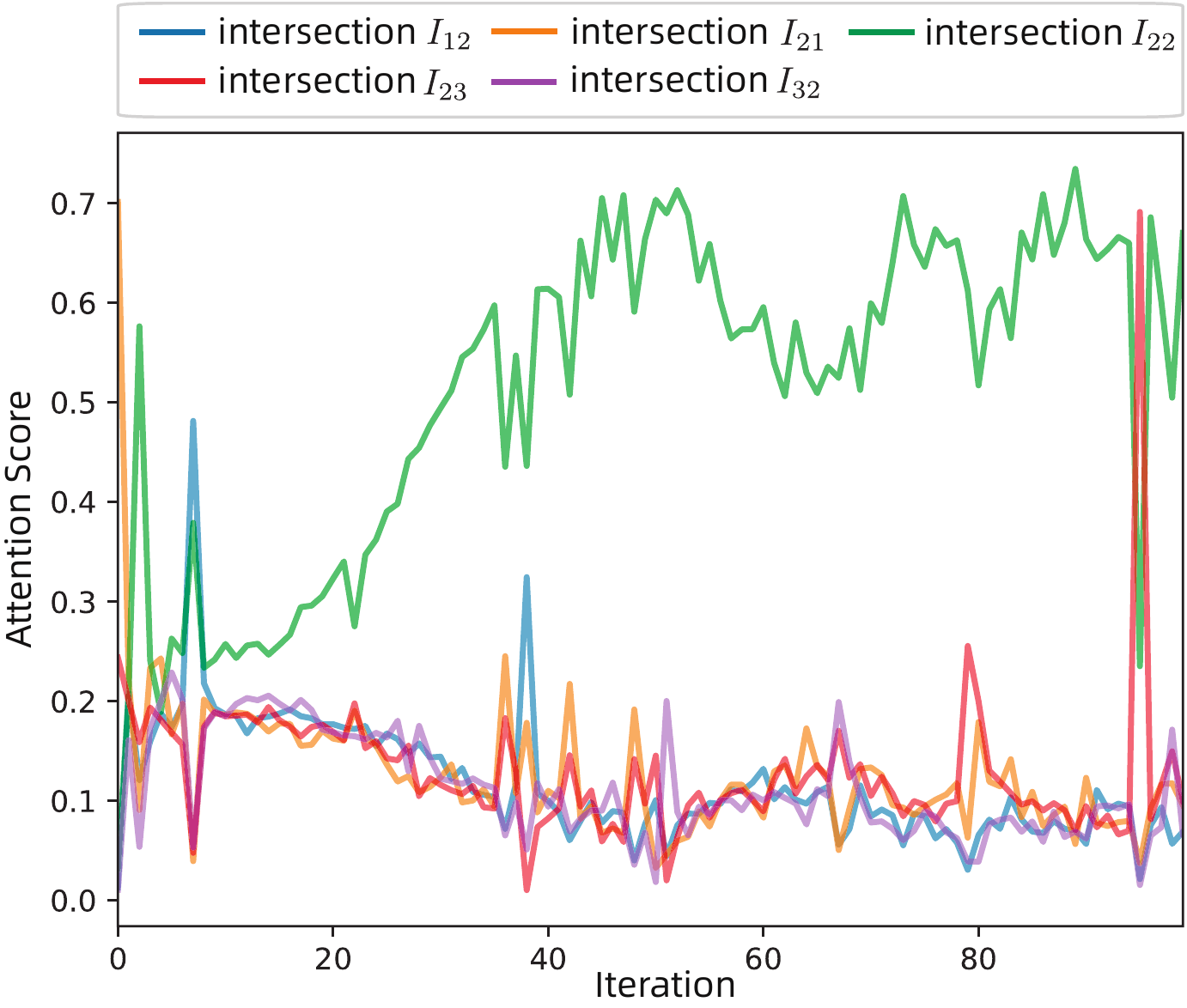}}
			\end{minipage}    
		}
		\subfigure[Attention scores of intersection $I_{22}$ selected from $ Grid_{3\times 3}uni $]{
			\begin{minipage}[b]{0.8\linewidth}
				\centerline{\includegraphics[width=1\linewidth]{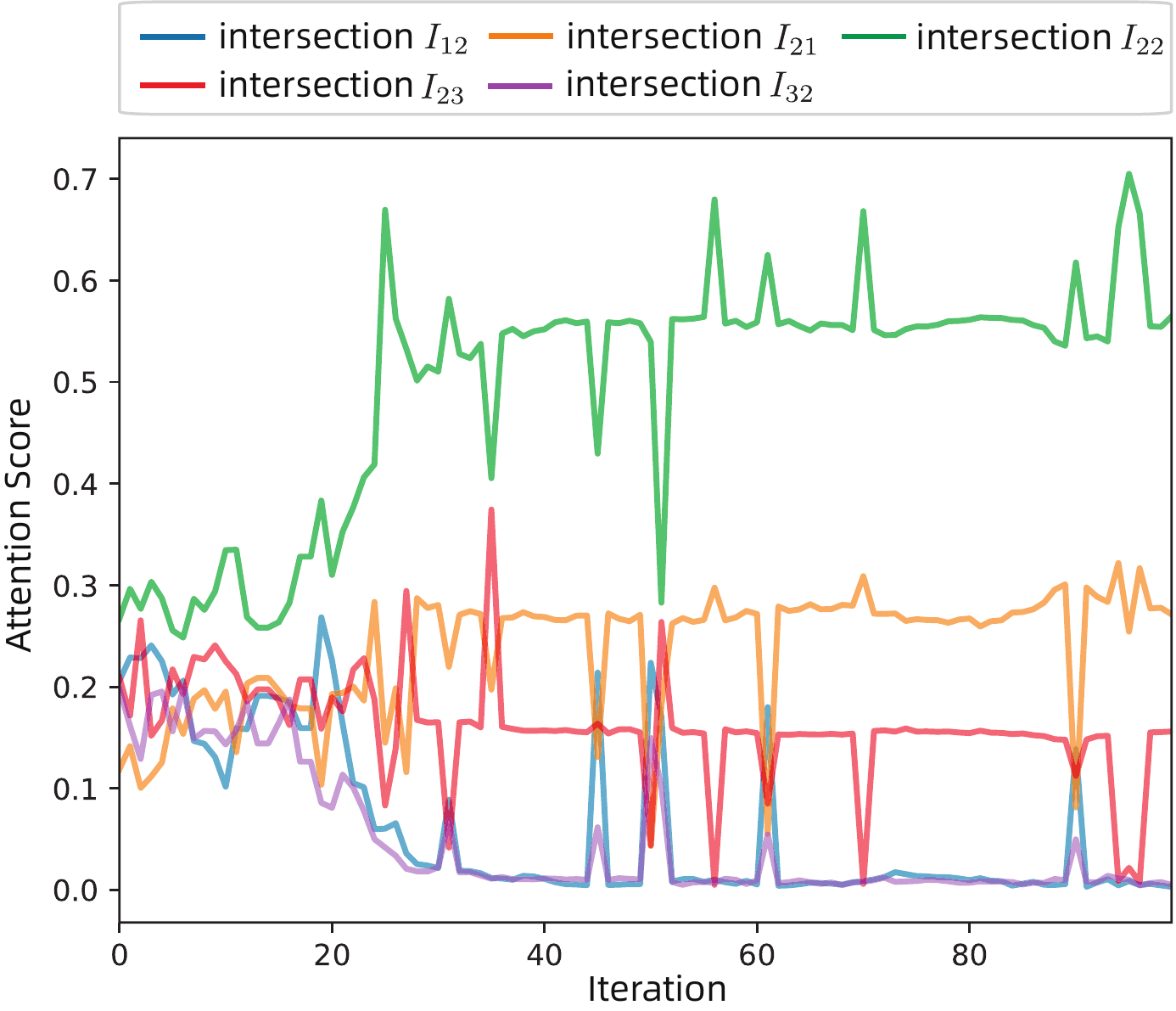}}
			\end{minipage}
		}
		\caption{Attention layers learn the different of surrounding from traffic flow; $ Grid_{3\times 3}bi $ and $ Grid_{3\times 3}uni $ road network are used for comparison to show the effect of attention layers.}
		\label{attention_score}
	\end{figure}
	
	Figure \ref{attention_score} shows the comparison between one-way and two-way $ 3\times 3 $ network.
	Comparing Figure \ref{attention_score}(a) with Figure \ref{attention_score}(b), the score change of intersection $ I_{22} $ in Figure \ref{attention_score}(a) is the equalization of surrounding intersection $ I_{12} $, $ I_{21} $, $ I_{23} $ and $ I_{32} $. While in Figure \ref{attention_score}(b), the intersection from the direction of exiting is obviously holding a commanding edge. 
	This means that the attention mechanism does have a significant effect in understanding the structure of the road network. 
	With the attention mechanism, the reward amendment is more concerned with the results of its actions, which is crucial in the revision of reward.

	\section{Discussion about Real World Application}
	\begin{figure}[htbp]
		\centerline{\includegraphics[width=1\linewidth]{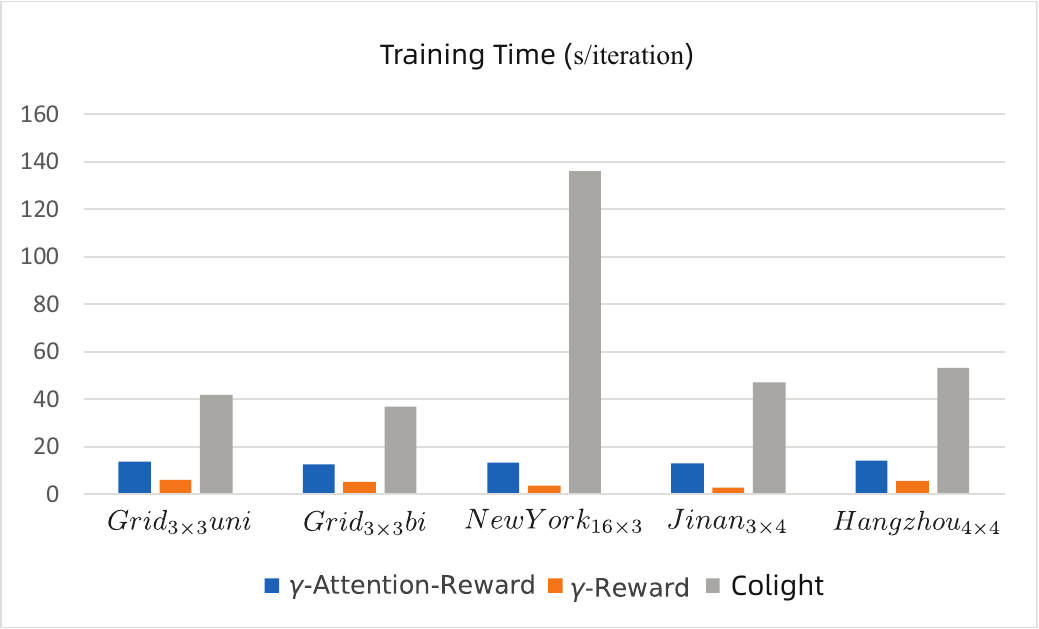}}
		\caption{Training time of \textit{$ \gamma $-Reward} series and Colight; Note that in order to show the difference between centralized and decentralized algorithms, we illustrate training time of one agent in \textit{$ \gamma $-Reward} series.}
		\label{training time}
	\end{figure}
	TSC is an actual task running without the simulator, and the proposed control plan should aim at solving practical problems. Therefore, it is necessary to take the limitations that may exist in practice into account. 
	
	\subsection{Scalability}
	Scalability is an important standard to evaluate whether a TSC method is meaningful. The proposed methods in this paper are both decentralized, each D3QN agent can be deployed on the embedding device in each intersection respectively. Either the training time or the inference time of our methods is not related to the number of intersections. Besides, the message transmission among neighboring intersections just takes tolerable time and all of the intersections communicate in parallel. Due to the limitation of computing capacity, experiments of larger scale are not able to conduct. Instead, we analyze the scalability by collecting the training time on datasets with different scales which is listed in Figure \ref{training time}. As the number of intersections increased, the time cost of Colight is significantly larger. In contrast, our methods remain a shorter training time due to the parallel computation. Another advantage is when the road network structure is changed, we only need to train the newly added intersections separately and there is a tiny impact on existing intersections. Therefore, the proposed methods are believed to be scalable.
	 
	\subsection{Real-time}
	Real-time traffic communication may have problems like communication delay, information security, and risk of packet loss\cite{ezell2013cyber}. Therefore, it is necessary to decouple the calculation of the whole road network. Decentralization is becoming the trend of solutions in TSC problem\cite{liu2017distributed}\cite{chu2016large}. Colight, while using Multi-head Attention technology and sharing parameter, summarizes the relationship of the global road network and reduces the time complexity and space complexity of training. 
	However, if we apply it to the practice, the global road network information in real-time is first needed to transmit to the central server for calculation, and then the server needs to dispense the resulting actions to each intersection. Note that it is not just transmitting actions of each traffic signal, but also their sensor observations! In contrast, communication in our methods occurs in a small neighborhood, and the messages they need are only several bits. These attributes help the intersections keep in touch in real-time.

	\subsection{Further applications}
	The coordination mechanism in this paper can be expended to other multi-agent tasks which face the problem of scalability and real-time.  For instance, the power dispatching of power networks can be described as the same as the TSC problem. Power transmission is similar to the traffic flow between two neighboring intersections. Furthermore, the control of a multi-joint tandem robot may be regarded as a multi-agent task. Actions from the joint which nears the base have a consequence on other joints and the end-effecter. We can use the coordination mechanism to consider this effect on neighboring agents or joints and achieve a global optimal performance. There are already broad applications using this kind of coordination mechanism in engineering systems like sensor networks\cite{rabbat2004distributed}, smart grid\cite{zhang2018dynamic} and robotics\cite{corke2005networked}.
	
	\section{Conclusion}
	In this paper, we propose the \textit{$\gamma$-Reward} method and its variant \textit{$\gamma$-Attention-Reward} that introduces the attention mechanism to solve the problem of intelligent control of the traffic signal. By extending the Markov Chain to the space-time domain, the methods turn to be a scalable solution for TSC. Specially, we give a concise proof of the \textit{spatial differentiation} formula which shows that the two frameworks can converge to Nash equilibrium. We conduct extensive experiments using both real-world and synthetic data. They confirm that our proposed methods have a superior performance over state-of-the-art methods. In the asymmetry road network, \textit{$\gamma$-Attention-Reward} shows inspired results than \textit{$\gamma$-Reward} by adding the attention mechanism. Moreover, we investigate the effect of the reward amendment and attention mechanism in achieving coordination thoroughly.
	Compared to the recently proposed Colight, \textit{$\gamma$-Reward} series replaces the graph attention with recursion, decouples the calculation of the whole road networks, and is more suitable for practical applications.

	\section*{Acknowledgment}
	
	This project is derived from the 46th project of Deecamp 2019, which won the Best Technology Award, and we thanks for the effort of our teammates in “Multi-Commander”\footnote{Our project of Deecamp 2019 is open-source in Github: https://github.com/multi-commander/Multi-Commander}.
	Thanks to the practice platform provided by Dr. Kai-Fu Lee and the support of APEX Lab of Shanghai Jiaotong University. We also gratefully acknowledge the support from the National Natural Science Foundation of China (Grant No. 61973210), the projects of SJTU (Grant Nos. YG2019ZDA17; ZH2018QNB23).
	
	\bibliographystyle{ieeetr}
	\bibliography{ref}

\begin{thebibliography}{10}

\bibitem{to2001white}
H.~R. TO and M.~M. Barker, ``White paper european transport policy for 2010:
  time to decide,'' 2001.

\bibitem{Liang}
X.~Liang, X.~Du, G.~Wang, and Z.~Han, ``Deep reinforcement learning for traffic
  light control in vehicular networks,'' {\em arXiv preprint arXiv:1803.11115},
  2018.

\bibitem{wei2018intellilight}
H.~Wei, G.~Zheng, H.~Yao, and Z.~Li, ``Intellilight: A reinforcement learning
  approach for intelligent traffic light control,'' in {\em Proceedings of the
  24th ACM SIGKDD International Conference on Knowledge Discovery \& Data
  Mining}, pp.~2496--2505, ACM, 2018.

\bibitem{DBLP:journals/corr/abs-1904-08117}
H.~Wei, G.~Zheng, V.~V. Gayah, and Z.~Li, ``A survey on traffic signal control
  methods,'' {\em CoRR}, vol.~abs/1904.08117, 2019.

\bibitem{Varaiya2013The}
P.~Varaiya, ``The max-pressure controller for arbitrary networks of signalized
  intersections,'' in {\em Advances in Dynamic Network Modeling in Complex
  Transportation Systems}, pp.~27--66, Springer, 2013.

\bibitem{bucsoniu2010multi}
L.~Bu{\c{s}}oniu, R.~Babu{\v{s}}ka, and B.~De~Schutter, ``Multi-agent
  reinforcement learning: An overview,'' in {\em Innovations in multi-agent
  systems and applications-1}, pp.~183--221, Springer, 2010.

\bibitem{zhang2019multi}
K.~Zhang, Z.~Yang, and T.~Ba{\c{s}}ar, ``Multi-agent reinforcement learning: A
  selective overview of theories and algorithms,'' {\em arXiv preprint
  arXiv:1911.10635}, 2019.

\bibitem{Wei2019}
H.~Wei, N.~Xu, H.~Zhang, G.~Zheng, X.~Zang, C.~Chen, W.~Zhang, Y.~Zhu, K.~Xu,
  and Z.~Li, ``Colight: Learning network-level cooperation for traffic signal
  control,'' in {\em Proceedings of the 28th ACM International Conference on
  Information and Knowledge Management}, pp.~1913--1922, 2019.

\bibitem{koonce2008traffic}
P.~Koonce and L.~Rodegerdts, ``Traffic signal timing manual.,'' tech. rep.,
  United States. Federal Highway Administration, 2008.

\bibitem{T1996The}
J.~T{\"o}r{\"o}k and J.~Kert{\'e}sz, ``The green wave model of two-dimensional
  traffic: Transitions in the flow properties and in the geometry of the
  traffic jam,'' {\em Physica A: Statistical Mechanics and its Applications},
  vol.~231, no.~4, pp.~515--533, 1996.

\bibitem{mousavi2017traffic}
S.~S. Mousavi, M.~Schukat, and E.~Howley, ``Traffic light control using deep
  policy-gradient and value-function-based reinforcement learning,'' {\em IET
  Intelligent Transport Systems}, vol.~11, no.~7, pp.~417--423, 2017.

\bibitem{xiong2019learning}
Y.~Xiong, G.~Zheng, K.~Xu, and Z.~Li, ``Learning traffic signal control from
  demonstrations,'' in {\em Proceedings of the 28th ACM International
  Conference on Information and Knowledge Management}, pp.~2289--2292, 2019.

\bibitem{zheng2019learning}
G.~Zheng, Y.~Xiong, X.~Zang, J.~Feng, H.~Wei, H.~Zhang, Y.~Li, K.~Xu, and
  Z.~Li, ``Learning phase competition for traffic signal control,'' in {\em
  Proceedings of the 28th ACM International Conference on Information and
  Knowledge Management}, pp.~1963--1972, 2019.

\bibitem{foerster2016learning}
J.~Foerster, I.~A. Assael, N.~de~Freitas, and S.~Whiteson, ``Learning to
  communicate with deep multi-agent reinforcement learning,'' in {\em Advances
  in Neural Information Processing Systems}, pp.~2137--2145, 2016.

\bibitem{Tan}
M.~Tan, ``Multi-agent reinforcement learning: Independent vs. cooperative
  agents,'' in {\em Proceedings of the tenth international conference on
  machine learning}, pp.~330--337, 1993.

\bibitem{rashid2018qmix}
T.~Rashid, M.~Samvelyan, C.~Schroeder, G.~Farquhar, J.~Foerster, and
  S.~Whiteson, ``Qmix: Monotonic value function factorisation for deep
  multi-agent reinforcement learning,'' in {\em International Conference on
  Machine Learning}, pp.~4295--4304, 2018.

\bibitem{van2016coordinated}
E.~Van~der Pol and F.~A. Oliehoek, ``Coordinated deep reinforcement learners
  for traffic light control,'' {\em Proceedings of Learning, Inference and
  Control of Multi-Agent Systems (at NIPS 2016)}, 2016.

\bibitem{nishi2018traffic}
T.~Nishi, K.~Otaki, K.~Hayakawa, and T.~Yoshimura, ``Traffic signal control
  based on reinforcement learning with graph convolutional neural nets,'' in
  {\em 2018 21st International Conference on Intelligent Transportation Systems
  (ITSC)}, pp.~877--883, IEEE, 2018.

\bibitem{wei2019presslight}
H.~Wei, C.~Chen, G.~Zheng, K.~Wu, V.~Gayah, K.~Xu, and Z.~Li, ``Presslight:
  Learning max pressure control to coordinate traffic signals in arterial
  network,'' in {\em Proceedings of the 25th ACM SIGKDD International
  Conference on Knowledge Discovery \& Data Mining}, pp.~1290--1298, 2019.

\bibitem{wiki:xxx}
{Wikipedia contributors}, ``Lane --- {Wikipedia}{,} the free encyclopedia,''
  2019.
\newblock [Online; accessed 1-November-2019].

\bibitem{stevanovic2010adaptive}
A.~Stevanovic, {\em Adaptive traffic control systems: domestic and foreign
  state of practice}.
\newblock No.~Project 20-5 (Topic 40-03), 2010.

\bibitem{mnih2015human}
V.~Mnih, K.~Kavukcuoglu, D.~Silver, A.~A. Rusu, J.~Veness, M.~G. Bellemare,
  A.~Graves, M.~Riedmiller, A.~K. Fidjeland, G.~Ostrovski, {\em et~al.},
  ``Human-level control through deep reinforcement learning,'' {\em Nature},
  vol.~518, no.~7540, p.~529, 2015.

\bibitem{VanHasselt2016}
H.~Van~Hasselt, A.~Guez, and D.~Silver, ``Deep reinforcement learning with
  double q-learning,'' in {\em Thirtieth AAAI conference on artificial
  intelligence}, 2016.

\bibitem{Wang2016}
Z.~Wang, T.~Schaul, M.~Hessel, H.~Hasselt, M.~Lanctot, and N.~Freitas,
  ``Dueling network architectures for deep reinforcement learning,'' in {\em
  International conference on machine learning}, pp.~1995--2003, 2016.

\bibitem{hessel2017rainbow}
M.~Hessel, J.~Modayil, H.~van Hasselt, T.~Schaul, G.~Ostrovski, W.~Dabney,
  D.~Horgan, B.~Piot, M.~G. Azar, and D.~Silver, ``Rainbow: Combining
  improvements in deep reinforcement learning,'' in {\em AAAI}, 2018.

\bibitem{liang2019deep}
X.~Liang, X.~Du, G.~Wang, and Z.~Han, ``A deep reinforcement learning network
  for traffic light cycle control,'' {\em IEEE Transactions on Vehicular
  Technology}, vol.~68, no.~2, pp.~1243--1253, 2019.

\bibitem{varshavskaya2009efficient}
P.~Varshavskaya, L.~P. Kaelbling, and D.~Rus, ``Efficient distributed
  reinforcement learning through agreement,'' in {\em Distributed Autonomous
  Robotic Systems 8}, pp.~367--378, Springer, 2009.

\bibitem{kar2013cal}
S.~Kar, J.~M. Moura, and H.~V. Poor, ``$\mathcal{QD}$-learning: A collaborative
  distributed strategy for multi-agent reinforcement learning through consensus
  + innovations,'' {\em IEEE Transactions on Signal Processing}, vol.~61,
  no.~7, pp.~1848--1862, 2013.

\bibitem{liu2017distributed}
W.~Liu, G.~Qin, Y.~He, and F.~Jiang, ``Distributed cooperative reinforcement
  learning-based traffic signal control that integrates v2x networks’ dynamic
  clustering,'' {\em IEEE Transactions on Vehicular Technology}, vol.~66,
  no.~10, pp.~8667--8681, 2017.

\bibitem{Bahdanau2014}
D.~Bahdanau, K.~Cho, and Y.~Bengio, ``Neural machine translation by jointly
  learning to align and translate,'' {\em arXiv preprint arXiv:1409.0473},
  2014.

\bibitem{Cho2014}
K.~Cho, B.~Van~Merri{\"e}nboer, C.~Gulcehre, D.~Bahdanau, F.~Bougares,
  H.~Schwenk, and Y.~Bengio, ``Learning phrase representations using rnn
  encoder-decoder for statistical machine translation,'' {\em arXiv preprint
  arXiv:1406.1078}, 2014.

\bibitem{Vaswani2017}
A.~Vaswani, N.~Shazeer, N.~Parmar, J.~Uszkoreit, L.~Jones, A.~N. Gomez,
  {\L}.~Kaiser, and I.~Polosukhin, ``Attention is all you need,'' in {\em
  Advances in neural information processing systems}, pp.~5998--6008, 2017.

\bibitem{kipf2016semi}
T.~N. Kipf and M.~Welling, ``Semi-supervised classification with graph
  convolutional networks,'' {\em arXiv preprint arXiv:1609.02907}, 2016.

\bibitem{velivckovic2017graph}
P.~Veli{\v{c}}kovi{\'c}, G.~Cucurull, A.~Casanova, A.~Romero, P.~Li{\`o}, and
  Y.~Bengio, ``Graph attention networks,'' in {\em International Conference on
  Learning Representations}, 2018.

\bibitem{Iqbal2018}
S.~Iqbal and F.~Sha, ``Actor-attention-critic for multi-agent reinforcement
  learning,'' {\em arXiv preprint arXiv:1810.02912}, 2018.

\bibitem{omidshafiei2017deep}
S.~Omidshafiei, J.~Pazis, C.~Amato, J.~P. How, and J.~Vian, ``Deep
  decentralized multi-task multi-agent reinforcement learning under partial
  observability,'' in {\em Proceedings of the 34th International Conference on
  Machine Learning-Volume 70}, pp.~2681--2690, JMLR. org, 2017.

\bibitem{laurent2011world}
G.~J. Laurent, L.~Matignon, L.~Fort-Piat, {\em et~al.}, ``The world of
  independent learners is not markovian,'' {\em International Journal of
  Knowledge-based and Intelligent Engineering Systems}, vol.~15, no.~1,
  pp.~55--64, 2011.

\bibitem{aragon2020traffic}
R.~Aragon-G{\'o}mez and J.~B. Clempner, ``Traffic-signal control reinforcement
  learning approach for continuous-time markov games,'' {\em Engineering
  Applications of Artificial Intelligence}, vol.~89, p.~103415, 2020.

\bibitem{nguyen2014decentralized}
D.~T. Nguyen, W.~Yeoh, H.~C. Lau, S.~Zilberstein, and C.~Zhang, ``Decentralized
  multi-agent reinforcement learning in average-reward dynamic dcops,'' in {\em
  Twenty-Eighth AAAI Conference on Artificial Intelligence}, 2014.

\bibitem{hu2003nash}
J.~Hu and M.~P. Wellman, ``Nash q-learning for general-sum stochastic games,''
  {\em Journal of machine learning research}, vol.~4, no.~Nov, pp.~1039--1069,
  2003.

\bibitem{Zhang2019CityFlow}
H.~Zhang, S.~Feng, C.~Liu, Y.~Ding, Y.~Zhu, Z.~Zhou, W.~Zhang, Y.~Yu, H.~Jin,
  and Z.~Li, ``Cityflow: A multi-agent reinforcement learning environment for
  large scale city traffic scenario,'' in {\em The World Wide Web Conference},
  pp.~3620--3624, ACM, 2019.

\bibitem{Krajzewicz2010Traffic}
D.~Krajzewicz, ``Traffic simulation with sumo--simulation of urban mobility,''
  in {\em Fundamentals of traffic simulation}, pp.~269--293, Springer, 2010.

\bibitem{Moritz2018}
P.~Moritz, R.~Nishihara, S.~Wang, A.~Tumanov, R.~Liaw, E.~Liang, M.~Elibol,
  Z.~Yang, W.~Paul, M.~I. Jordan, {\em et~al.}, ``Ray: A distributed framework
  for emerging ai applications,'' in {\em 13th USENIX Symposium on Operating
  Systems Design and Implementation (OSDI 18)}, pp.~561--577, 2018.

\bibitem{ezell2013cyber}
B.~C. Ezell, R.~M. Robinson, P.~Foytik, C.~Jordan, and D.~Flanagan, ``Cyber
  risk to transportation, industrial control systems, and traffic signal
  controllers,'' {\em Environment Systems and Decisions}, vol.~33, no.~4,
  pp.~508--516, 2013.

\bibitem{chu2016large}
T.~Chu, S.~Qu, and J.~Wang, ``Large-scale traffic grid signal control with
  regional reinforcement learning,'' in {\em 2016 American Control Conference
  (ACC)}, pp.~815--820, IEEE, 2016.

\bibitem{rabbat2004distributed}
M.~Rabbat and R.~Nowak, ``Distributed optimization in sensor networks,'' in
  {\em Proceedings of the 3rd international symposium on Information processing
  in sensor networks}, pp.~20--27, 2004.

\bibitem{zhang2018dynamic}
K.~Zhang, W.~Shi, H.~Zhu, E.~Dall’Anese, and T.~Ba{\c{s}}ar, ``Dynamic power
  distribution system management with a locally connected communication
  network,'' {\em IEEE Journal of Selected Topics in Signal Processing},
  vol.~12, no.~4, pp.~673--687, 2018.

\bibitem{corke2005networked}
P.~Corke, R.~Peterson, and D.~Rus, ``Networked robots: Flying robot navigation
  using a sensor net,'' in {\em Robotics research. The eleventh international
  symposium}, pp.~234--243, Springer, 2005.

\bibitem{Sutton1988}
R.~S. Sutton, ``Learning to predict by the methods of temporal differences,''
  {\em Machine learning}, vol.~3, no.~1, pp.~9--44, 1988.

\bibitem{Sutton2018}
R.~S. Sutton and A.~G. Barto, {\em Reinforcement learning: An introduction}.
\newblock MIT press, 2018.

\bibitem{watkins1992q}
C.~J. Watkins and P.~Dayan, ``Q-learning,'' {\em Machine learning}, vol.~8,
  no.~3-4, pp.~279--292, 1992.

\end{thebibliography}
	
	\newpage
	\section*{Appendix}
	\subsection{Fundamental DRL algorithms}
	\subsubsection{Temporal-Difference Learning}
	Temporal-Difference Learning (TD learning), proposed by Sutton\cite{Sutton1988}, combining with Dynamic Programming (DP) and Monte Carlo (MC) methods, becomes the core idea of DRL. Like the Monte Carlo algorithm, it does not need to know the specific environment model and can learn directly from experience. On the other hand, it also inherits bootstrapping from  DP algorithm, which is the unique feature of TD learning: predictions are used as targets during learning\cite{Sutton2018}. Monte Carlo simulates (or experiences) an episode until it ends, then estimates the state value based on the value of each state. In contrast, TD learning simulates an episode with one step (or several steps) per action which based on the value of the new state, and then estimate the state value before execution. 
	
	The Q-learning algorithm \cite{watkins1992q} is a groundbreaking algorithm. 
	TD learning is used here for off-policy learning.
	$$\delta_t= r_{t+1}+\gamma' \max _{a} Q\left(o_{t+1}, a_{t+1}\right)-Q\left(o_{t}, a_{t}\right)$$
	where $\delta_t$ represents TD-error.
	
	\subsubsection{Deep Q-network}
	Deep Q-network (DQN) is a powerful off-policy algorithm which has achieved excellent results in many fields since 2015\cite{mnih2015human}. It uses a neural network to approximate the Q-value function instead of tabling.
	$$
	Q\left(o_{t}, a_{t}\right) \leftarrow Q\left(o_{t}, a_{t}\right)+\alpha\left[y_j-Q\left(o_{t}, a_{t}\right)\right]
	$$
	$$
	y_{t+1}=r_{t+1}+\gamma ' \max _{a} Q\left(o_{t+1}, a_{t+1}\right)
	$$
	\subsubsection{Double Deep Q-network}
	Q-learning uses \textit{max} to select the best action, which causes a Maximization Bias problem. So \cite{VanHasselt2016} solved this by designing a Double Q-learning, it only differs in the calculation of the target Q value:
	$$
	y_{t+1}=r_{t+1}+\gamma' Q^{\prime}\left(\phi\left(o^{\prime}\right), \arg \max _{a^{\prime}} Q\left(\phi\left(o^{\prime}\right), a, w\right), w^{\prime}\right)
	$$
	\subsubsection{Dueling Deep Q-network}
	Another improvement for DQN is Dueling DQN \cite{Wang2016}, which decomposes the Q network into two separate control streams, a value function $ V(s) $, and a state-based action advantage function $ A(s, a) $. These two control flows obtain an estimate value of the Q function through a special aggregating layer.
	$$
	\begin{aligned} Q(o, a ; \theta, \alpha, \beta) &=V(o ; \theta, \beta)+\\ &\left(A(o, a ; \theta, \alpha)-\frac{1}{|\mathcal{A}|} \sum_{a^{\prime}} A\left(o, a^{\prime} ; \theta, \alpha\right)\right) \end{aligned}
	$$
	\newpage
	\subsection{Pseudocode for $ \gamma $-Reward series}
	\begin{algorithm}[h]
		\caption{\textit{$ \gamma $-Attention-Reward} Algorithm for MARL Traffic Lights Control}
		\label{alg::conjugateGradient}
		\begin{algorithmic}[1]
			\State Initialize $ E $ parallel environments with $ N $ agents
			\State Initialize replay memory  $ D $  to capacity  $ N_D $
			\State Initialize raw replay memory  $ D_r $  to capacity  $ N_D+n $
			\State Initialize action-value function  $ Q $ with random weights  $ \theta $
			\State Initialize target action-value function $  \hat{Q} $ with weights $  \theta^{-}=\theta $
			\State Initialize attention scores $ \alpha_{i,j} $
			\State $ T_{update} \leftarrow 0 $
			\For{$episode =1,  M $}
			\State Reset environments, and get initial $ o_i $ for each agent $ i $
			\For{$t=1, T $}
			\parState{%
				Select actions $ a_i\sim \pi_i(\cdot|o_i) $ for each agent $ i $ in each environment $ e $}
			\parState{%
				Send actions $ a_i $ to all parallel environments and get $ o_i' $ , $ r_i $ for all agents}
			\State Store ($ a_i, o_i,r_i,o_i'  $) in $ D_r $
			\State $ T_{update}=T_{update}+N $
			\If{$ T_{update} \leq$ min steps per update + $ n $}
			\State \textit{Replay Buffer Amendment($ D_r, \alpha_{i, j}$)}
			\parState{%
				Update Policies:\\
				Calculate $a_{1 \ldots N}^{B} \sim \pi_{i}^{\bar{\theta}}\left(o_{i}^{\prime B}\right), i \in 1 \ldots N$\\ 
				Calculate $Q_{i}^{\psi}\left(o_{1 \ldots N}^{B}, a_{1 \ldots N}^{B}\right)$ for all $i$ in parallel\\ 
				Update policies using $ \nabla_{\theta,i}\mathcal{J}(\pi_{\theta}) $ and Adam 
			}

			\State Update target $Q$ parameters: $  \hat{\theta} \leftarrow \theta $
			\State Update target attention parameters: $  \hat{\alpha}\leftarrow \alpha $
			\State $ T_{update} \leftarrow 0 $
			\EndIf
			\EndFor
			\EndFor
		\end{algorithmic}
	\end{algorithm}
	
	\begin{algorithm}[h]
		\caption{\textit{Replay Buffer Amendment}}
		\label{alg::conjugateGradient}
		\begin{algorithmic}[1]
			\Function{ReplayBufferAmendment}{$D_r, \alpha_{i,j}$}
			\For{$ i =1, N $}
			\State $ index= len(D_r)-n$
			\While{$ index> ex\_index$}
			\State $(o_i,a_i,r_i,o_i') =D_{r,i}(j)$
			\State $ R_i = \gamma $\textit{-Attention-Reward}  function($ r_i , \alpha_{i,j}$)
			
			\State Store ($ o_i,a_i,R_i,o_i'$) in $ D $
			\State    $ j=j-1 $
			\EndWhile
			\State $ ex\_index=len(D_r)-n $
			\EndFor
			
			\EndFunction
			
		\end{algorithmic}
	\end{algorithm}
\end{document}